\documentclass{article} %
\pdfoutput=1  %
\usepackage{iclr2018_conference,times}

\let\ifarxiv\iftrue  %

\usepackage[utf8]{inputenc}
\usepackage[T1]{fontenc}
\usepackage[american]{babel}

\usepackage{amsmath}
\usepackage{amssymb}
\usepackage{amsthm}
\usepackage{graphicx}
\usepackage{url}
\usepackage{booktabs}
\usepackage{mathtools}
\usepackage{subcaption}
\usepackage{enumitem}
\usepackage{xpatch}

\usepackage[colorlinks,linkcolor={red!80!black},citecolor={green!80!black},urlcolor={blue!80!black},hypertexnames=false]{hyperref}

\usepackage[capitalize,noabbrev]{cleveref}
\usepackage{autonum}

\newcommand{\httpurl}[1]{\href{http://#1}{\nolinkurl{#1}}}
\newcommand{\httpsurl}[1]{\href{https://#1}{\nolinkurl{#1}}}

\DeclareMathOperator{\E}{\mathbb E}

\newcommand{\F}{\mathcal F}

\newcommand{\h}{\mathcal H}
\DeclareMathOperator{\K}{\mathit{K}}
\DeclareMathOperator{\IPM}{\mathit{IPM}}
\DeclareMathOperator{\MMD}{\mathit{MMD}}
\DeclareMathOperator{\FID}{\mathit{FID}}
\DeclareMathOperator{\relu}{ReLU}
\newcommand{\M}{\mathcal M}
\newcommand{\n}{\mathbb N}
\newcommand{\N}{\mathcal N}
\newcommand{\R}{\mathbb R}
\newcommand{\tp}{^\mathsf{T}}
\DeclareMathOperator{\tr}{Tr}
\newcommand{\ud}{\mathrm{d}}
\newcommand{\X}{\mathcal X}
\newcommand{\Y}{\mathcal Y}

\newcommand{\Xset}{\mathbf{X}}
\newcommand{\Yset}{\mathbf{Y}}
\newcommand{\Zset}{\mathbf{Z}}

\newcommand{\PP}{\mathbb P}
\newcommand{\Q}{\mathbb Q}
\newcommand{\Z}{\mathbb Z}
\newcommand{\bP}{\PP}
\newcommand{\bQ}{\Q}
\newcommand{\QQ}{\Q}

\DeclareMathOperator*{\argmax}{argmax}

\DeclareMathOperator{\D}{\mathcal D}
\DeclareMathOperator{\W}{\mathcal W}

\newcommand{\Xfull}{\Xset}
\newcommand{\Yfull}{\Yset}

\newcommand{\Xtr}{\mathbf{X}^\mathit{tr}}
\newcommand{\xtr}{X^\mathit{tr}}
\newcommand{\ntr}{n^\mathit{tr}}
\newcommand{\Ytr}{\mathbf{Y}^\mathit{tr}}
\newcommand{\ytr}{Y^\mathit{tr}}
\newcommand{\mtr}{m^\mathit{tr}}

\newcommand{\Xte}{\mathbf{X}^\mathit{te}}
\newcommand{\nte}{n^\mathit{te}}
\newcommand{\Yte}{\mathbf{Y}^\mathit{te}}
\newcommand{\mte}{m^\mathit{te}}

\newcommand{\RBF}{\TextOrMath{\textit{rbf}}{\mathit{rbf}}}
\newcommand{\RQ}{\TextOrMath{\textit{rq}}{\mathit{rq}}}
\newcommand{\DIST}{\TextOrMath{\textit{dist}}{\mathit{dist}}}
\newcommand{\DOT}{\TextOrMath{\textit{dot}}{\mathit{dot}}}

\DeclarePairedDelimiterX{\KLx}[2]{(}{)}{#1\;\delimsize\|\;#2}
\newcommand{\KL}{\mathrm{KL}\KLx}

\newtheorem{theorem}{Theorem}
\newtheorem{corollary}{Corollary}
\newtheorem{prop}{Proposition}
\newtheorem{lemma}{Lemma}
\newtheorem{definition}{Definition}
\newlist{assumplist}{enumerate}{1}
\setlist[assumplist]{label=\textbf{\Alph*}}
\Crefname{assumplisti}{Assumption}{Assumptions}

\makeatletter
\newcommand{\labelsuffix}[2]{%
    \addtocounter{\@listctr}{-1}%
    \refstepcounter{\@listctr}%

    \edef\@currentlabel{\@currentlabel#2}%

    \global\def\cref@currentlabel{[\@listctr][\arabic{\@listctr}][]\@currentlabel}%

    \expandafter\def\expandafter\label@listctr\expandafter{\csname label\@listctr\endcsname}%
    \expandafter\def\expandafter\c@@listctr\expandafter{\csname c@\@listctr\endcsname}%

    \expandafter\expandafter\expandafter\expandafter\expandafter\expandafter\expandafter%
    \xpatchcmd\expandafter\expandafter\expandafter%
    \label@listctr\expandafter\expandafter\expandafter%
        {\expandafter\expandafter\expandafter%
            {\expandafter%
                \c@@listctr\expandafter%
            }\expandafter%
        }\expandafter%
        {\expandafter%
            {\c@@listctr}#2%
        }{}{err}%

    \csname label\@listctr\endcsname%
    \label{#1}%
}
\makeatother

\title{Demystifying MMD GANs}

\author{Miko\l{}aj Bi\'nkowski\thanks{These authors contributed equally.}\\
Department of Mathematics\\
Imperial College London\\
\texttt{mikbinkowski@gmail.com} \\
\AND
Danica J. Sutherland\footnotemark[1]{}, Michael Arbel \& Arthur Gretton\\
Gatsby Computational Neuroscience Unit\\
University College London\\
\texttt{\{danica.j.sutherland,michael.n.arbel,arthur.gretton\}@gmail.com}
}

\iclrfinalcopy %

\begin{document}

\maketitle

\begin{abstract}
  We investigate the training and performance of generative adversarial networks using the Maximum Mean Discrepancy (MMD) as critic, termed MMD GANs.
  As our main theoretical contribution,
  we clarify the situation with bias in GAN loss functions raised by recent work:
  we show that gradient estimators used in the optimization process for both MMD GANs and Wasserstein GANs are unbiased,
  but learning a discriminator based on samples leads to biased gradients for the generator parameters.
  We also discuss the issue of kernel choice for the MMD critic, and characterize the kernel corresponding to the energy distance used for the Cram\'er GAN critic.
  Being an integral probability metric, the MMD benefits from training strategies recently developed for Wasserstein GANs. In experiments,  the MMD GAN is able to employ a smaller critic network than the Wasserstein GAN, resulting in a simpler and faster-training
  algorithm with matching performance.
  We also propose an improved measure of GAN convergence, the \emph{Kernel Inception Distance}, and show how to use it to dynamically adapt learning rates during GAN training.
\end{abstract}

\section{Introduction} \label{sec:intro}

Generative Adversarial Networks \citep[GANs;][]{gans} provide a powerful  method for general-purpose generative modeling of datasets.
Given examples from some distribution,
a GAN attempts to learn a \emph{generator} function,
which maps from some fixed noise distribution to samples that attempt to mimic a reference or target distribution.
The generator is trained to trick a \emph{discriminator}, or \emph{critic},
which tries to distinguish between generated and target samples.

This alternative to standard maximum likelihood approaches for training generative models has brought about a rush of interest over the past several years.
Likelihoods do not necessarily correspond well to sample quality \citep{likelihoods-vs-samples},
and GAN-type objectives focus much more on producing plausible samples,
as illustrated particularly directly by \citet{gan-vs-ml-realnvp}.
This class of models has recently led to many impressive examples of image generation
\citep[e.g.][]{beyond-face-rotation,stacked-gans,anime-gans,cyclegan}.

GANs are, however, notoriously tricky to train \citep{improved-gans}.
This might be understood in terms of the discriminator class.
\citet{gans} showed that,
when the discriminator is trained to optimality among a rich enough function class,
the generator network attempts to minimize the Jensen-Shannon divergence between the generator and target distributions. This result has been extended to general $f$-divergences by \citet{NowBotRyo16}.
According to \citet{towards-principled-gans}, however,
it is likely that both the GAN and reference probability measures are supported on manifolds within a larger space,
as occurs for the set of images in the space of possible pixel values.
These manifolds might not intersect at all, or at best might intersect on sets of measure zero.
In this case, the Jensen-Shannon divergence is constant, and the KL and reverse-KL divergences are infinite,
meaning that they provide no useful gradient for the generator to follow. This helps to explain some of the instability of GAN training.

The lack of sensitivity to distance,
meaning that nearby but non-overlapping regions of high probability mass are not considered similar,
is a long-recognized problem for KL divergence-based discrepancy measures \citep[e.g.][Section 4.2]{GneRaf07}.
It is natural to address this problem using Integral Probability Metrics \citep[IPMs;][]{Mueller97}:
these measure the distance between probability measures via the largest discrepancy in expectation over a class of ``well behaved'' witness functions. Thus, IPMs are able to signal proximity in the probability mass of the generator and reference distributions.
(\Cref{sec:losses} describes this framework in more detail.)

\citet{wgan} proposed to use the Wasserstein distance between  distributions as the discriminator,
which is an integral probability metric constructed from the witness class of 1-Lipschitz functions. To implement the Wasserstein critic,
\citeauthor{wgan} originally proposed weight clipping of the discriminator network, to enforce  $k$-Lipschitz smoothness.
\citet{wgan-gp} improved on this result by directly constraining the gradient of the discriminator network at points between the generator and reference samples. This new Wasserstein GAN implementation, called WGAN-GP, is more stable and easier to train.

A second integral probability metric used in GAN variants is the maximum mean discrepancy (MMD), for which the witness function class is a unit ball in a reproducing kernel Hilbert space (RKHS).
Generative adversarial models based on minimizing the MMD were
first considered by \citet{gmmn} and \citet{gen-mmd}.
These works optimized a generator to minimize the MMD with a fixed kernel,
either using a generic kernel on image pixels
or by modeling autoencoder representations instead of images directly.
\citet{opt-mmd} instead minimized the statistical power of an MMD-based  test with a fixed kernel.
Such approaches struggle with complex  natural images,
where pixel distances are of little value,
and fixed representations can easily be tricked,
as in the adversarial examples of \citet{adversarial}.

Adversarial training of the MMD loss is thus an obvious choice to advance these methods.
Here the kernel MMD is defined on the output of a convolutional network, which is trained adversarially.
Recent notable work has made use of the IPM representation of the MMD to employ the same witness function
regularization strategies as \citet{wgan} and \citet{wgan-gp},
effectively corresponding to an additional constraint on the MMD function class.
Without such constraints, the convolutional
features are unstable and difficult to train \citep{opt-mmd}.
\cite{mmd-gan} essentially used the weight clipping strategy of \citeauthor{wgan}, with additional constraints to encourage
the kernel distribution embeddings to be injective.\footnote{When distribution embeddings are injective, the critic is guaranteed to be able to distinguish any two distributions, given an infinite number of samples.}
In light of the observations by \citeauthor{wgan-gp}, however, we use a gradient constraint on the MMD witness
function in the present work (see \cref{sec:mmd,sec:grad-penalty}).\footnote{\citeauthor{mmd-gan} also did this in a later revision of their paper, independent of this work.}
\citet{cramer-gan}'s method, the Cram\'er GAN,  also used the gradient constraint strategy of \citeauthor{wgan-gp} in their  discriminator network. As we discuss in \cref{sec:cramer}, the Cram\'er GAN discriminator is related to the energy distance, which is an instance of the MMD \citep{energy-dist-is-mmd}, and which can therefore use a gradient constraint on the witness function. Note, however, that there are important differences between the Cram\'er GAN critic and the energy distance, which make it more akin to the optimization of a scoring rule: we provide further details in \cref{sec:score-funcs}.
Weight clipping and gradient constraints
are not the only approaches possible: variance features \citep{mcgan} and constraints \citep{fisher-gan}
can work, as can other optimization strategies \citep{began,dan}.

Given that both the Wasserstein distance and the MMD are integral probability metrics, it is of interest to consider how they differ when used in GAN training.
\citet{cramer-gan} showed that optimizing the empirical Wasserstein distance can lead to biased gradients for the generator,
and gave an explicit example where optimizing with these biased gradients leads the optimizer to incorrect parameter values, even in expectation.
They then claim that the energy distance does not suffer from these problems.
As our main theoretical contribution, we substantially clarify the bias situation in \cref{sec:unbiased-gradients}.
First, we show (\cref{thm:main-body}) that the natural maximum mean discrepancy estimator, including the estimator of energy distance,
has unbiased gradients when used ``on top'' of a fixed deep network representation.
The generator gradients obtained from a trained representation, however, will be biased relative to the desired gradients of the optimal critic based on infinitely many samples.
This situation is exactly analogous to WGANs:
the generator's gradients with a fixed critic are unbiased,
but gradients from a learned critic are biased with respect to the supremum over critics.

MMD GANs, though, do have some advantages over Wasserstein GANs.
Certainly we would not expect the MMD on its own to perform well on raw image data, since these data lie on a low dimensional manifold embedded in a higher dimensional pixel space.
Once the images are mapped through appropriately trained convolutional layers, however, they can follow a much simpler distribution with broader support across the mapped domain: a phenomenon also observed in autoencoders \citep{BenMesDauRif13}.
In this setting, the MMD with characteristic kernels \citep{SriGreFukLanetal10} shows strong discriminative performance between distributions.
To achieve comparable performance, a WGAN without the advantage of a kernel on the transformed space requires many more convolutional filters in the critic.
In our experiments (\cref{sec:experiments}),
we find that MMD GANs achieve the same generator performance as WGAN-GPs with smaller discriminator networks,
resulting in GANs with fewer parameters and computationally faster training.
Thus, the MMD GAN discriminator can be understood as a hybrid model that plays to the strengths of both the initial convolutional mappings and the kernel layer that sits on top.

\section{Losses and witness functions} \label{sec:losses}

We begin with a review of the MMD
and relate it to the loss functions used by other GAN variants.
Through its interpretation as an integral probability metric, we show
that the gradient penalty of \citet{wgan-gp}
applies to the MMD GAN.

\subsection{Maximum Mean Discrepancy and witness functions} \label{sec:mmd}

We consider a random variable $X$ with probability measure $\bP$,
which we associate with the generator, and a second random variable
$Y$ with probability measure $\bQ$, which we associate with the reference
sample that we wish to learn.
Our goal is to measure the distance from $\bP$ to $\bQ$ using
samples drawn independently from each distribution.

The maximum mean discrepancy is a metric on probability measures \citep{mmd-jmlr},
which falls within the family of integral probability metrics \citep{Mueller97};
this family includes the Wasserstein and Kolmogorov metrics, but
not for instance the KL or $\chi^{2}$ divergences.
Integral probability metrics make use of a class of \emph{witness functions} to distinguish between $\bP$ and $\bQ$, choosing the function with the largest discrepancy in expectation over $\bP,\bQ$,
\[
\D_\F(\bP, \bQ) = \sup_{f \in \F} \E_{\bP} f(X) - \E_{\bQ} f(Y)
\label{eq:IPM}
.\]
The particular witness function class $\F$
determines the probability metric.\footnote{We assume throughout that if $f \in \F$, we also have $-f \in \F$, so that $\D_\F$ is symmetric.}
For example,
the Wasserstein-1 metric is defined using the 1-Lipschitz functions,
the total variation by functions with absolute value bounded by 1,
and the Kolmogorov metric using the functions of bounded variation $1$.
For more on this family of distances, see e.g.\ \citet{ipms-phi-clf}.

In this work, our witness function class $\F$ will be the unit ball in
a reproducing kernel Hilbert space  $\h$, with positive
definite kernel $k(x, x')$.
The key aspect of a reproducing kernel Hilbert space is the reproducing property:
for all $f \in \h$, $f(x) = \left\langle f, k(x,\cdot) \right\rangle_{\h}$.
We define the  \emph{mean embedding} of the probability measure $\bP$
as the element $\mu_{\bP} \in \h$
such that $\E_{\bP} f(X) = \left\langle f, \mu_{\bP} \right\rangle_\h$;
it is given by $\mu_\bP = \E_{X \sim \bP} k(\cdot, X)$.%
\footnote{This is well defined for Bochner-integrable kernels \citep[Definition A.5.20]{SteChr08},
for which $E_{\bP}\left\Vert k(x,\cdot)\right\Vert _\h < \infty$
for the class of probability measures $\bP$ being considered.
For bounded kernels, the condition always holds, but for unbounded kernels, additional conditions on the moments might apply.}

The maximum mean discrepancy (MMD) is defined as the IPM \eqref{eq:IPM} with $\F$ the unit ball in $\h$,  %
\[
\MMD(\bP, \bQ; \h)  = \sup_{f\in\h, \lVert f \rVert_\h \le 1} \E_{\bP} f(X) - \E_{\bQ} f(Y)
\label{eq:mmd}
.\]
The witness function $f^{*}$ that attains the supremum has a straightforward
expression \citep[Section 2.3]{mmd-jmlr},
\[
f^{*}(x)\propto \E_{\bP} k(X,x) - \E_{\bQ} k(Y,x). \label{eq:witness}
\]
Given
samples $X = \{x_{i}\}_{i=1}^{m}$ drawn i.i.d.\ from $\bP$, and
$Y = \{y_{j}\}_{j=1}^{n}$ drawn i.i.d.\ from $\bQ$,
the empirical witness function is
\[
\hat{f}(x)\propto\frac{1}{m}\sum_{i=1}^{m}k(x_{i},x)-\frac{1}{n}\sum_{i=1}^{n}k(y_{i},x), \label{eq:empiricalWitness}
\]
and an unbiased estimator of the squared MMD is \citep[Lemma 6]{mmd-jmlr}
\begin{equation}
\MMD_u^{2}(X,Y) =
    \frac{1}{m(m-1)} \sum_{i \ne j}^{m} k(x_{i}, x_{j})
    + \frac{1}{n(n-1)} \sum_{i \ne j}^{n} k(y_{i}, y_{j})
    - \frac{2}{mn} \sum_{i=1}^{m} \sum_{j=1}^{n} k(x_{i}, y_{j})
.\label{eq:unbiasedMMD}
\end{equation}
When the kernel is characteristic \citep{SriGreFukLanetal10,SriFukLan11},
the embedding $\mu_{\bP}$ is injective (i.e., associated uniquely with
$\bP$).
Perhaps the best-known characteristic kernel is the exponentiated quadratic kernel, also known as the Gaussian RBF kernel,
\[
    k^\RBF_\sigma(x, y) = \exp\left( - \frac{1}{2 \sigma^2} \lVert x - y \rVert^2 \right)
\label{eq:k-rbf}
.\]
Both the kernel and its derivatives decay exponentially, however,
causing significant problems in high dimensions, and especially when used in gradient-based representation learning.
The rational quadratic kernel
\[
k^\RQ_\alpha(x, y) = \left( 1 + \frac{\lVert x - y \rVert^2}{2 \alpha} \right)^{-\alpha}
\label{eq:k-rq}
\]
with $\alpha > 0$ corresponds to a scaled mixture of exponentiated quadratic
kernels,
with a $\mathrm{Gamma}(\alpha, 1)$ prior on the inverse lengthscale \citep[Section 4.2]{RasWil06}.
This kernel will be the mainstay of our experiments,
as its tail behaviour is much superior to that of the exponentiated quadratic kernel;
it is also characteristic.

\subsection{Witness function and gradient penalties} \label{sec:grad-penalty}

The MMD has been a popular choice for the role of a critic in a GAN.
This idea was proposed simultaneously by \citet{gen-mmd} and \citet{gmmn}, with numerous
recent follow-up works \citep{opt-mmd,energy-distance-gan,mmd-gan,cramer-gan}.
As a key strategy in these recent works,
the MMD of (\ref{eq:unbiasedMMD}) is not computed directly on the samples;
rather,
the samples first pass through a mapping function $h$,
generally a convolutional network.
Note that we can think of this either as the MMD with kernel $k$ on features $h(x)$,
or simply as the MMD with kernel $\kappa(x, y) = k(h(x), h(y))$.
The challenge is to learn the features $h$ so as to maximize the MMD, without
causing the critic to collapse to a trivial answer early in training.

Bearing in mind that the MMD is an integral probability metric, strategies
developed for training the Wasserstein GAN critic can be directly adopted for training the MMD critic.
\citet{mmd-gan} employed the weight clipping approach of \citet{wgan},
though they motivated it using different considerations.
\citet{wgan-gp} found a number of issues with weight clipping, however:
it oversimplifies the loss functions given standard architectures,
the gradient decays exponentially as we move up the network,
and it seems to require the use of slower optimizers such as RMSProp
rather than standard approaches such as Adam \citep{adam}.

It thus seems preferable to adopt \citeauthor{wgan-gp}'s proposal of regularising the critic witness (\ref{eq:empiricalWitness})
by constraining its gradient norm to be nearly 1
along randomly chosen convex combinations of generator and reference points,
$\alpha x_i +(1-\alpha) y_j$ for $\alpha \sim \mathrm{Uniform}(0,1)$.
This was motivated by the observation that the Wasserstein witness satisfies this property (their Lemma 1),
but perhaps its main benefit is one of regularization:
if the critic function becomes too flat anywhere between the samples,
the generator cannot easily follow its gradient.
We will thus follow this approach,
as did \citet{cramer-gan},
whose model we describe next.\footnote{By doing so, we implicitly change the definition of the distance being approximated; we leave study of the differences to future work. By analogy, \citet{approx-convergence-props} give some basic properties for the distance used by \citet{wgan-gp}.}

\subsection{The energy distance and associated MMD}\label{sec:cramer}

\citet{energy-distance-gan} and \citet[Section 4]{cramer-gan}
proposed to use the energy distance as the critic in an adversarial network.
The energy distance \citep{SzeRiz04,Lyons13} is a measure of divergence
between two probability measures, defined as
\[
  \D_{e}(\bP,\bQ) =
    - \frac{1}{2} \E_{\bP} \rho(X,X')
    - \frac{1}{2} \E_{\bQ} \rho(Y, Y')
    + \E_{\bP,\bQ} \rho(X,Y)
\label{eq:energyDistance}
,\]
where $\E_{\bP} \rho(X, X')$ is an expectation over two independent samples from the generator $\bP$
(likewise, $Y$ and $Y'$ are independent samples from the reference $\bQ$),
and $\rho(x,y)$ is a semimetric of negative type.%
\footnote{$\rho$ must satisfy the properties of a metric besides the triangle inequality, and for all $n \ge 2$, $x_1, \ldots, x_n \in \mathcal{X}$,
and $a_1, \dots, a_n \in \R$ with $\sum_i a_i = 0$,
it must hold that
$\sum_{i=1}^{n} \sum_{j=1}^{n} a_{i} a_{j} \rho(x_{i},x_{j}) \le 0$.
}
We will focus on the case $\rho_\beta(x, y)=\lVert x - y \rVert^{\beta}$
for $0<\beta\le2$. When $\beta \le 1$, $\rho_\beta$ is a metric.

\citet[Lemma 12]{energy-dist-is-mmd} showed that the energy
distance is an instance of the maximum mean discrepancy, where the
corresponding distance-induced kernel family for the distance (\ref{eq:energyDistance})
is
\[
    k^\DIST_{\rho,z_0}(x, y) = \frac12\left[ \rho(x, z_0) + \rho(y, z_0) - \rho(x,y) \right]
\label{eq:distanceKernel}
\]
for any choice of $z_{0} \in \mathcal{X}$.
\citep[Often $z_{0}=0$ is chosen to simplify notation,
which corresponds to a fractional Brownian motion kernel;][Example 15.]{energy-dist-is-mmd}
Note that the resulting kernel is \emph{not} translation invariant.
It is characteristic (when $\beta < 2$), and the MMD is well-defined
for a class of distributions $\mathcal{P}$ that satisfy a moment condition \citep[Remark 21]{energy-dist-is-mmd}.\footnote{%
Namely, there must exist $z_{0}\in\mathcal{X}$ such that
$\int \rho(z,z_{0}) \,\ud \bP(z) < \infty$ for all $\bP\in\mathcal{P}$.}

To apply the regularization strategy of \cite{wgan-gp} in training the critic of an adversarial network,
we need to compute the form taken by the witness function \eqref{eq:witness}
given the kernel \eqref{eq:distanceKernel}.
Bearing in mind that
\[
\E_\bP k(x,X) = \rho(x,z_{0}) + \E_\bP \rho(X, z_{0}) - \E_\bP \rho(x, X),
\]
where the second term of the above expression is constant, and substituting
into \eqref{eq:witness}, we have
\begin{align}
f^*(x)
  &\propto \rho(x,z_{0}) - \E_{\bP}\rho(x,X) - \rho(x,z_{0}) + \E_{\bQ}\rho(x,Y) + C
\\&= \E_\bQ \rho(x,Y) - \E_\bP \rho(x,X) + C
.\end{align}
This is in agreement with \citeauthor{cramer-gan}'s function $f^{*}$ (their page 5),
via a different argument (though note that the function $g^{*}$ in
their footnote 4 is missing the constant terms).

We now turn to the divergence implemented by the critic in \citeauthor{cramer-gan}'s Algorithm 1,
which is somewhat different from the energy distance (\ref{eq:energyDistance}).
The Cram\'er GAN witness function is defined as%
\begin{equation}
f_{c}(x) = \E_{\bP}\rho(x,X)-\rho(x,0),\label{eq:surrogateWitness}
\end{equation}
which is regularized using \citeauthor{wgan-gp}'s gradient constraint.
The expected \emph{surrogate loss} associated with this witness function,
and used for the Cram\'er critic, is
\[
\D_{c}(\bP,\bQ) =
    \E_{\bP}\rho(X, X')
    + \E_{\bQ}\rho(Y, 0)
    - \E_{\bP}\rho(X, 0)
    - \E_{\bP,\bQ}\rho(X', Y)
\label{eq:surrogate}
.\]
In brief, $Y'$ in (\ref{eq:energyDistance}) is replaced by the
origin: it is explained that this is necessary for instance in the
conditional case, where two independent reference samples $Y,Y'$ are
not available.
Unfortunately, following this change, it becomes straightforward to define $\bP$ and $\bQ$ which are
different, yet have an expected $\D_{c}(\bP,\bQ)$ loss of zero.
For example, if $\bP$ is a point mass at the origin in $\R$,
and $\bQ$ is a point mass a distance $t$ from the origin,
then $\bP \ne \bQ$
and yet
$\D_c(\bP, \bQ) = 0$ because
\[
\E_{\bP} \rho(X, X') = \E_{\bP}\rho(X,0) = 0,
\qquad
\E_{\bP,\bQ} \rho(X', Y) = \E_{\bQ} \rho(Y, 0) = t
.\]

Nevertheless, good empirical performance has been obtained in practice for the Cram\'er critic,
both by \citet{cramer-gan} and in our experiments of \cref{sec:experiments}.
Our \cref{sec:score-funcs} provides some insight into this behavior by considering the Cram\'er critic's relationship to the score function associated with the energy distance.

\subsection{Other related models}
Many other GAN variants fall into the framework of IPMs \citep[e.g.][]{mcgan,fisher-gan,began}.
Notably,
although \citet{gans} motivated GANs as estimating the Jensen-Shannon divergence,
they can also be viewed as minimizing the IPM defined by the classifier family \citep{arora:gen-equilibrium,approx-convergence-props},
thus motivating applying the gradient penalty to original GANs \citep{many-paths}.
\citet{approx-convergence-props} in particular study properties of these distances.

\section{Gradient bias} \label{sec:unbiased-gradients}

The issue of biased gradients in GANs was brought to prominence by \citet[Section 3]{cramer-gan},
who showed bias in the gradients of the empirical Wasserstein distance for finite sample sizes,
and demonstrated cases where this bias could lead to serious problems in stochastic gradient descent,
even in expectation.
They then claimed that the energy distance used in the Cram\'er GAN critic does not suffer from these problems.
We will now both formalize and clarify these results.

First, \citeauthor{cramer-gan}'s proof that the gradient of the energy distance is unbiased was incomplete:
the essential step in the reasoning, the exchange in the order of
the expectations and the derivatives, is simply assumed.%
\footnote{Suppose $\hat f(x)$ is an unbiased estimator of a function $f(x)$,
so that $\E \hat f(x) = f(x)$.
Then, if we can exchange expectations and gradients, it is immediate that
$\E \nabla \hat f(x) = \nabla \E \hat f(x) = \nabla f(x)$.}
We show that
one can exchange the order of expectations and derivatives,
under very mild assumptions about the distributions in question,
the form of the network, and the kernel:
\begin{theorem} \label{thm:main-body}
    Let $G_\psi : \mathcal Z \to \X$
    and $h_\theta : \X \to \R^d$
    be deep networks, with parameters $\psi \in \R^{m_\psi}$ and $\theta \in \R^{m_\theta}$,
    of the form defined in \cref{sec:proof:notation}
    and satisfying \cref{Lipschitz,Diff_set} (in \cref{sec:proof:assumptions}).
    This includes almost all feedforward networks used in practice,
    in particular covering convolutions, max pooling, and ReLU activations.

    Let $\PP$ be a distribution on $\X$ such that $\E[\lVert X \rVert^2]$ exists,
    and likewise $\Z$ a distribution on $\mathcal Z$ such that $\E[\lVert Z \rVert^2]$ exists.
    $\PP$ and $\Z$ need not have densities.

    Let $k : \R^d \times \R^d \to \R$ be a kernel function satisfying the growth assumption \cref{kernel_growth_pair} for some $\alpha \in [1, 2]$.
    All kernels considered in this paper satisfy this assumption; see the discussion after \cref{cor:mmd-gan}.

    For $\mu$-almost all $(\psi, \theta) \in \R^{m_\psi + m_\theta}$,
    where $\mu$ is the Lebesgue measure,
    the function
    \[
        (\psi, \theta) \mapsto \E_{\substack{X \sim \PP \\ Z \sim \Z}}\left[
            k(h_\theta(X), h_\theta(G_\psi(Z)))
        \right]
    \]
    is differentiable at $(\psi, \theta)$,
    and moreover
    \[
        \E_{\substack{X \sim \PP \\ Z \sim \Z}}\left[ \partial_{\psi,\theta}
            k(h_\theta(X), h_\theta(G_\psi(Z)))
        \right]
        = \partial_{\psi,\theta} \E_{\substack{X \sim \PP \\ Z \sim \Z}}\left[
            k(h_\theta(X), h_\theta(G_\psi(Z)))
        \right]
    .\]
    Thus for $\mu$-almost all $(\psi, \theta)$,
    \[
        \E_{\substack{\Xset \sim \PP^m \\ \Zset \sim \Z^n}}\left[ \partial_{\psi,\theta}
            \MMD_u^2(h_\theta(\Xset), h_\theta(G_\psi(\Zset)))
        \right]
        = \partial_{\psi,\theta} \left[
            \MMD_u^2(h_\theta(\PP), h_\theta(G_\psi(\Z)))
        \right]
        \label{eq:mmd-gan-unbiased}
    .\]
\end{theorem}
This result is shown in \cref{appendix:proof},
specifically as \cref{cor:mmd-gan} to \cref{thm:interversion},
which is a quite general result about interchanging expectations and derivatives of functions of deep networks.
The proof is more complex than a typical proof that derivatives and integrals can be exchanged,
due to the non-differentiability of ReLU-like functions used in deep networks.

But this unbiasedness result is not the whole story.
In WGANs, the generator attempts to minimize the loss function
\[
    \W(\PP, \Q) = \sup_{f : \lVert f \rVert_L \le 1} \E_{X \sim \PP} f(X) - \E_{Y \sim \Q} f(Y)
    \label{eq:sup-w}
,\]
based on an estimate $\hat\W(\Xset, \Yset)$:
first critic parameters $\theta$ are estimated on a ``training set'' $\Xtr$, $\Ytr$,
i.e.\ all points seen in the optimization process thus far,
and then the distance is estimated on the remaining ``test set'' $\Xte$, $\Yte$,
i.e.\ the current minibatch, as
\[
    \frac1{\mte} \sum_{i=1}^{\mte} f_\theta(X_i) - \frac1{\nte} \sum_{j=1}^{\nte} f_\theta(Y_j)
    \label{eq:j-hat}
.\]
(After the first pass through the training set, these two sets will not be quite independent, but for large datasets they should be approximately so.)
\Cref{thm:ipm-bias} (in \cref{sec:gen-ipm-bias}) shows that this estimator $\hat\W$ is biased; \cref{appendix:wgan-bias} further gives an explicit numerical example.
This almost certainly implies that $\nabla_\psi \hat\W$ is biased as well,
by \cref{thm:biased-means-biased-grad} (\cref{sec:gen-ipm-grad-bias}).\footnote{%
\citeauthor{cramer-gan}'s Appendix A.2 rather showed gradient bias in a different situation:
$\nabla \W(\hat\PP_m, \Q)$, where $\Q$ is a known distribution and $\hat\PP_m$ is the empirical distribution of $m$ samples from a distribution which changes as $m$ increases.}
Yet, for \emph{fixed} $\theta$,
\cref{cor:wgan} shows that the estimator \eqref{eq:j-hat} has unbiased gradients;
it is only the procedure which first selects a $\theta$ based on training samples and then evaluates \eqref{eq:j-hat}
which is a biased estimator of \eqref{eq:sup-w}.

The situation with MMD GANs, including energy distance-based GANs, is exactly analogous.
We have \eqref{eq:mmd-gan-unbiased}:
for almost all particular critic representations $h_\theta$,
the estimator of $\MMD^2$ is unbiased.
But the population divergence the generator attempts to minimize is actually
\[
    \eta(\PP, \Q) = \sup_{\theta} \MMD^2\left( h_\theta(\PP), h_\theta(\Q) \right)
    \label{eq:sup-mmd}
,\]
a distance previously studied by \citet{kernel-choice-mmd}
as well as \citet{mmd-gan}.
An MMD GAN's effective estimator of $\hat\eta$ is also biased by \cref{thm:ipm-bias}
(see particularly \cref{appendix:max-bias});
by \cref{thm:biased-means-biased-grad},
its gradients are also almost certainly biased.

In both cases, the bias vanishes as the selection of $\theta$ becomes better;
in particular, no bias is introduced by the use of a fixed (and potentially small) minibatch size,
but rather by the optimization procedure for $\theta$ and the total number of samples seen in training the discriminator.

Yet there is at least some sense in which MMD GANs might be considered ``less biased'' than WGANs.
Optimizing the generator parameters of a WGAN while holding the critic parameters fixed is not sensible:
consider, for example, $\PP$ a point mass at $0 \in \R$ and $\QQ$ a point mass at $q \in \R$.
If $q > 0$, an optimal $\theta$ might correspond to the witness function $f(t) = t$;
if we hold this witness function $f$ fixed, the optimal $q$ is at $-\infty$,
rather than at the correct value of $0$.
But if we hold an MMD GAN's critic fixed and optimize the generator,
we obtain the GMMN model \citep{gmmn,gen-mmd}.
Here, because the witness function still adapts to the observed pair of distributions,
the correct distribution $\PP = \QQ$ will always be optimal.
Bad solutions might also seem to be optimal,
but they can never seem arbitrarily \emph{better}.
Thus unbiased gradients of $\MMD_u^2$ might somehow be more meaningful to the optimization process
than unbiased gradients of \eqref{eq:j-hat};
exploring and formalizing this intuition is an intriguing area for future work.

\section{Evaluation metrics} \label{sec:evaluation}
One challenge in comparing GAN models, as we will do in the next section,
is that quantitative comparisons are difficult.
Some insight can be gained by visually examining samples,
but we also consider the following approaches to evaluate GAN methods.

\paragraph{Inception score}
This metric, proposed by \citet{improved-gans},
is based on the classification output $p(y \mid x)$ of the Inception model \citep{inception}.
Defined as $\exp\left( \E_x \KL{p(y \mid x)}{p(y)} \right)$,
it is highest when each image's predictive distribution has low entropy,
but the marginal predictive distribution $p(y) = \E_x p(y \mid x)$ has high entropy.
This score correlates somewhat with human judgement of sample quality on natural images,
but it has some issues,
especially when applied to domains which do not represent a variety of the types of classes in ImageNet.
In particular, it knows nothing about the desired distribution for the model.

\paragraph{FID}
The \emph{Fr\'echet Inception Distance}, proposed by \citet{fid},
avoids some of the problems of Inception
by measuring the similarity of the samples' representations in the Inception architecture (at the \texttt{pool3} layer, of dimension $2048$)
to those of samples from the target distribution.
The FID fits a Gaussian distribution to the hidden activations for each distribution
and then computes the Fr\'echet distance, also known as the Wasserstein-2 distance, between those Gaussians.
\citeauthor{fid} show that unlike the Inception score,
the FID worsens monotonically as various types of artifacts are added to CelebA images~--~though in our \cref{appendix:noise} we found the Inception score to be more monotonic than did \citeauthor{fid}, so this property may not be very robust to small changes in evaluation methods.
Note also that the estimator of FID is biased;\footnote{This is easily seen when the true FID is $0$: here the estimator may be positive, but can never be negative. Note also that in fact no unbiased estimator of the FID exists; see \cref{appendix:fid-bias:no-unbiased}.}
we will discuss this issue shortly.

\paragraph{KID}
We propose a metric similar to the FID,
the \emph{Kernel Inception Distance},
to be the squared MMD between Inception representations.
We use a polynomial kernel,
$k(x, y) = \left( \frac{1}{d} x\tp y + 1 \right)^3$ where $d$ is the representation dimension,
to avoid correlations with the objective of MMD GANs as well as to avoid tuning any kernel parameters.\footnote{$k$ is the default polynomial kernel in \texttt{scikit-learn} \citep{scikit-learn}.}
This can also be viewed as an MMD directly on input images with the kernel $K(x, y) = k(\phi(x), \phi(y))$,
with $\phi$ the function mapping images to Inception representations.
Compared to the FID,
the KID has several advantages.
First, it does not assume a parametric form for the distribution of activations.
This is particularly sensible since the representations have ReLU activations,
and therefore are not only never negative,
but do not even have a density:
about 2\% of components in Inception representations are typically exactly zero.
With the cubic kernel we use here,
the KID compares skewness as well as the mean and variance.
Also, unlike the FID, %
the KID has a simple unbiased estimator.\footnote{Because the computation of the MMD estimator scales like $O(n^2 d)$, we recommend using a relatively small $n$ and averaging over several estimates; this is closely related to the block estimator of \citet{b-test}. The FID estimator, for comparison, takes time $O(n d^2 + d^3)$, and is substantially slower for $d = 2048$.}
It also shares the behavior of the FID as artifacts are added to images (\cref{appendix:noise}).

\begin{figure}
    \begin{subfigure}[t]{.48\textwidth}
        \centering
        \includegraphics[width=\linewidth]{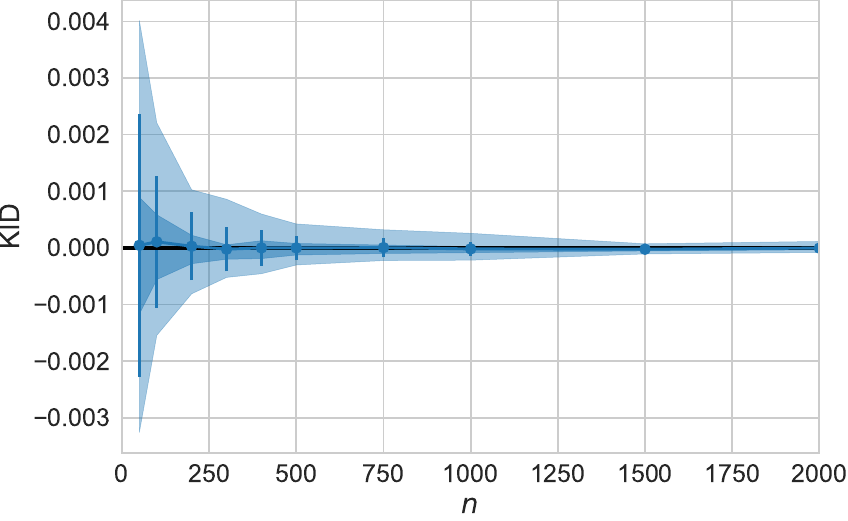}
        \caption{KID estimates are unbiased, and standard deviations shrink quickly even for small $n$.}
        \label{fig:kid-unbias}
    \end{subfigure}
    ~
    \begin{subfigure}[t]{.48\textwidth}
        \centering
        \includegraphics[width=\linewidth]{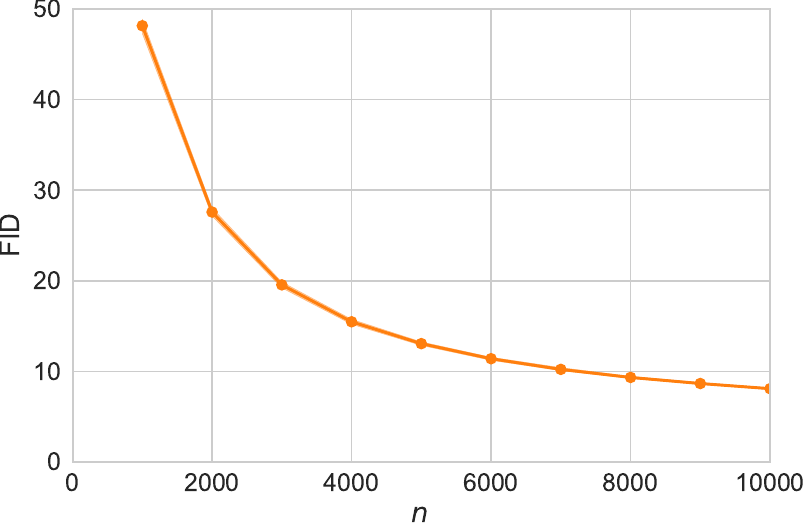}
        \caption{FID estimates exhibit strong bias for $n$ even up to $10\,000$. All standard deviations are less than $0.5$.}
        \label{fig:fid-bias}
    \end{subfigure}
    \caption{Estimates of distances between the CIFAR-10 train and test sets. Each point is based on 100 samples, estimating with replacement; sampling without replacement, and/or using the full training set, gives similar results. Lines show means, error bars standard deviations, dark colored regions a $\frac{2}{3}$ coverage interval of the samples, light colored regions a $95\%$ interval. Note the differing $n$ axes.}
    \label{fig:fid-kid-bias}
\end{figure}

\Cref{fig:fid-kid-bias} demonstrates the empirical bias of the FID and the unbiasedness of the KID by comparing the CIFAR-10 train and test sets.
The KID (\cref{fig:kid-unbias}) converges quickly to its presumed true value of 0;
even for very small $n$, simple Monte Carlo estimates of the variance provide a reasonable measure of uncertainty.
By contrast, the FID estimate (\cref{fig:fid-bias}) does not behave so nicely:
at $n = 2\,000$, when the KID estimator is essentially always 0,
the FID estimator is still quite large.
Even at $n = 10\,000$, the full size of the CIFAR test set,
the FID still seems to be decreasing from its estimate of about 8.1
towards zero, showing the strong persistence of bias.
This highlights that FID scores can only be compared to one another with the same value of $n$.

Yet even for the same value of $n$, there is no particular reason to think that the bias in the FID estimator will be the same when comparing different pairs of distributions.
In \cref{appendix:fid-bias},
we demonstrate two situations where $FID(\bP_1, \bQ) < FID(\bP_2, \bQ)$,
but for insufficent numbers of samples the estimator usually gives the other ordering.
This can happen even where all distributions in question are one-dimensional Gaussians, as \cref{appendix:fid-bias:1d-normal} shows analytically.
\Cref{appendix:fid-bias:relu} also empirically demonstrates this on distributions more like the ones used for FID in practice,
giving a simple example with $d = 2048$ where even estimating with $n = 50\,000$ samples reliably gives the wrong ordering between the models.
Moreover, Monte Carlo estimates of the variance are extremely small
even when the estimate is very far from its asymptote,
so it is difficult to judge the reliability of an estimate,
and practitioners may be misled by the very low variance into thinking that they have obtained the true value.
Thus comparing FID estimates bears substantial risks.
KID estimates, by contrast, are unbiased and asymptotically normal.

For models on MNIST, we replace the Inception featurization with features from a LeNet-like convolutional classifier\footnote{\scriptsize\httpsurl{github.com/tensorflow/models/blob/master/tutorials/image/mnist/convolutional.py}} \citep{lenet},
but otherwise compute the scores in the same way.

We also considered the diagnostic test of \citet{birthday-test},
which estimates the approximate number of ``distinct'' images produced by a GAN.
The amount of subjectivity in what constitutes a duplicate image,
however, makes it hard to reliably compare models based on this diagnostic.
Comparisons likely need to be performed both with a certain notion of duplication in mind
and by a user who does not know which models are being compared,
to avoid subconscious biases;
we leave further exploration of this intriguing procedure to future work.

\subsection{Learning rate adaptation} \label{sec:lr-adaptation}
In supervised deep learning, it is common practice to dynamically reduce the learning rate of an optimizer when it has stopped improving the metric on a validation set.
So far, this does not seem to be common in GAN-type models,
so that learning rate schedules must be tuned by hand.
We propose instead using an adaptive scheme,
based on comparing the KID score for samples from a previous iteration to that from the current iteration.

To avoid setting an explicit threshold on the change in the numerical value of the score,
we use a $p$-value obtained from the relative similarity test of \citet{3sample}.
If the test does not indicate that our current model is closer to the validation set than the model from a certain number of iterations ago at a given significance level, we mark it as a failure; when a given number of failures occur in a row,
we decrease the learning rate.
\citeauthor{3sample}'s test is for the hypothesis $\MMD(\bP_1, \bQ) < \MMD(\bP_2, \bQ)$,
and since the KID can be viewed as an MMD on image inputs, we can apply it directly.%
\footnote{We use the slight corrections to the asymptotic distribution of the MMD estimator given by \citet{opt-mmd} in this test.}

\section{Experiments} \label{sec:experiments}

We compare the quality of samples generated by MMD GAN using various kernels with samples obtained by WGAN-GP \citep{wgan-gp} and Cram\'er GAN \citep{cramer-gan} on four standard benchmark datasets:
the MNIST dataset of $28\times 28$ handwritten digits\footnote{\scriptsize\httpurl{yann.lecun.com/exdb/mnist/}},
the CIFAR-10 dataset of $32 \times 32$ photos \citep{cifar10},
the LSUN dataset of bedroom pictures resized to $64\times 64$ \citep{lsun},
and the CelebA dataset of celebrity face images resized and cropped to $160\times 160$ \citep{celeba}.

For most experiments, except for those with the CelebA dataset, we used the DCGAN architecture \citep{dcgan} for both generator and critic.
For MMD losses, we used only 16 top-layer neurons in the critic; more did not seem to improve performance,
except for the \emph{distance} kernel for which 256 neurons in the top layer was advantageous.
As
\citet{cramer-gan} advised to use at least 256-dimensional critic output, this enabled exact comparison
between Cram\'er GAN and energy distance MMD, which are directly related (\cref{sec:cramer}).
For the generator we used the standard number of convolutional filters (64 in the second-to-last layer);
for the critic, we compared networks with 16 and 64 filters in the first convolutional layer.\footnote{In
the DCGAN architecture the number of filers doubles in each consecutive layer, so an $f$-filter critic
has $f$, $2f$, $4f$ and $8f$ convolutional filters in layers 1-4, respectively.}

For the higher-resolution model for the CelebA dataset, we used a 5-layer DCGAN critic and a 10-layer ResNet
generator\footnote{As in \citet{wgan-gp}, we use a linear layer, 4 residual blocks and one convolutional layer.}, with 64
convolutional filters in the last/first layer. This allows us to compare the performance of MMD GANs with a more complex architecture.

Models with smaller critics run considerably faster:
on our systems, the 16-filter DCGAN networks typically ran at about twice the speed of the
64-filter ones.
Note that the critic size is far more important to training runtime than the generator size:
we update the critic 5 times for each generator step,
and moreover the critic network is run on two batches each time we use it,
one from $\bP$ and one from $\bQ$.
Given the same architecture, all models considered here run at about the same speed.

We evaluate several MMD GAN kernel functions in our experiments.\footnote{Because these higher-resolution experiments were slower to run, for CelebA we trained MMD GAN with only one type of kernel.}
The simplest is the linear kernel:
$k^\DOT(x, y) = \langle x, y \rangle$,
whose MMD corresponds  to the distance between means
\citep[this is somewhat similar to the \emph{feature matching} idea of ][]{improved-gans}.
We also use the exponentiated quadratic \eqref{eq:k-rbf} and rational quadratic \eqref{eq:k-rq} functions,
with mixtures of lengthscales,
\[
    k^{\RBF}(x, y) = \sum_{\sigma \in \Sigma} k_{\sigma}^{\RBF}(x, y),
    \qquad
    k^{\RQ}(x, y) = \sum_{\alpha \in \mathcal A} k^{\RQ}_{\alpha}(x, y)
,\]
where $\Sigma = \{2, 5, 10, 20, 40, 80\}$,
$\mathcal A = \{.2, .5, 1, 2, 5\}$. For the latter, however, we found it advantageous to add
a linear kernel to the mixture, resulting in the mixed \emph{RQ-dot} kernel $k^{\RQ*} = k^{\RQ} + k^{\DOT}$.
Lastly we use the distance-induced kernel $k^\DIST_{\rho_1,0}$ of \eqref{eq:distanceKernel},
using the Euclidean distance $\rho_1$ so that the MMD is the energy distance.%
\footnote{%
  We also found it helpful to add an activation penalty to the critic representation network in certain MMD models.
  Otherwise the representations $h_\theta$ sometimes chose very large values,
  which for most kernels does not change the theoretical loss (defined only in terms of distances)
  but leads to floating-point precision issues. We use a combined $L^2$ penalty on activations across all critic layers, with a factor of $1$ for $\RQ{}*$ and $0.0001$ for $\DIST$.
}
We also considered Cram\'er GANs, with the surrogate critic \eqref{eq:surrogate},
and WGAN-GPs.

Each model was trained with a batch size of 64, and 5 discriminator updates per generator update.
For CIFAR-10, LSUN and CelebA we trained for $150\,000$ generator updates,
while for MNIST we used $50\,000$.
The initial learning rate was set to $10^{-4}$ and followed the adaptive scheme described in \cref{sec:lr-adaptation}, with KID compared between the current model and the model $20\,000$ generator steps earlier ($5\,000$ for MNIST), every $2\,000$ steps ($500$ for MNIST). After 3 consecutive failures to improve, the learning rate was halved. This approach allowed us to avoid manually picking a different learning rate for each of the considered models.

We scaled the gradient penalty by $1$, instead of the $10$ recommended by \citet{wgan-gp} and \citet{cramer-gan}; we found this to usually work slightly better with MMD models. With the distance kernel, however, we scale the penalty by $10$ to allow direct comparison with Cram\'er GAN.

Quantitative scores are estimated based on $25\,000$ generator samples ($100\,000$ for MNIST),
and compared to $25\,000$ dataset elements (for LSUN and CelebA)
or the standard test set ($10\,000$ images held out from training for MNIST and CIFAR-10).
Inception and FID scores were computed using 10 bootstrap resamplings of the given images;
the KID score was estimated based on 100 repetitions of sampling $1\,000$ elements without replacement.

Code for our models is available at \httpsurl{github.com/mbinkowski/MMD-GAN}.

\paragraph{MNIST}
All of the models achieved good results,
measured both visually and in quantitative scores;
full results are in \cref{appendix:samples}.
\Cref{fig:mnist:traces}, however, shows the evolution of our quantitative criteria throughout the training process for several models.
This shows that the linear kernel \DOT{} and rbf kernel \RBF{} are clearly worse than the other models at the beginning of the training process,
but both improve eventually.
\RBF{}, however, never fully catches up with the other models.
There is also some evidence that \DIST{}, and perhaps WGAN-GP, converge more slowly than \RQ{} and Cram\'er GAN.
Given their otherwise similar properties,
we thus recommend the use of \RQ{} kernels over \RBF{} in MMD GANs and limit experiments for other datasets to \RQ{} and \DIST{} kernels.

\begin{figure}[h!]
    \centering
    \includegraphics[width=\linewidth]{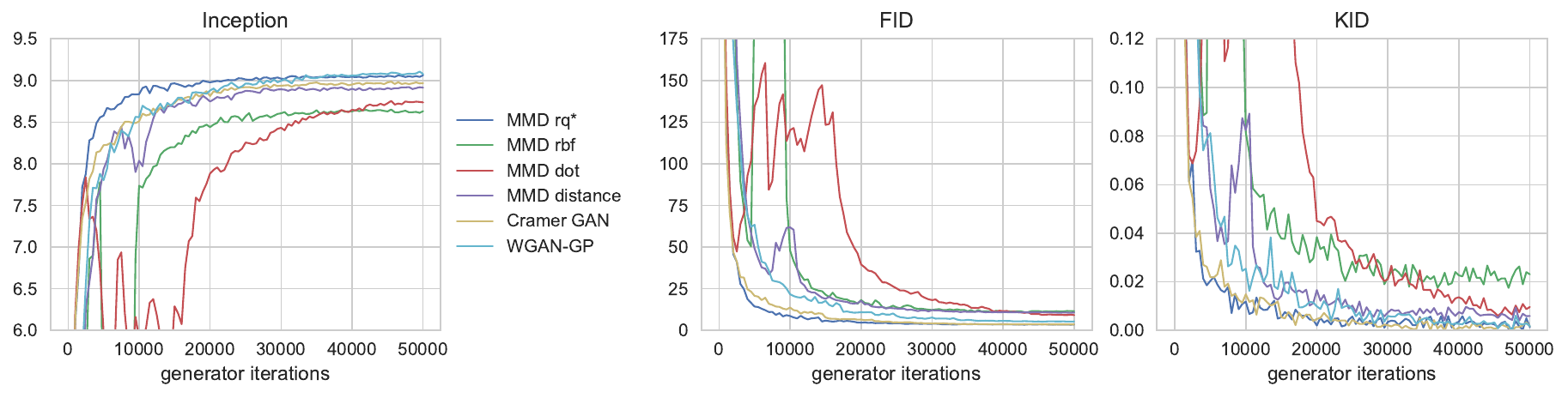}
    \caption{Score estimates over the learning process for MNIST training. }
    \label{fig:mnist:traces}
\end{figure}

\paragraph{CIFAR-10}
Full results are shown in \cref{appendix:samples}.
Small-critic MMD GAN models approximately match large-critic WGAN-GP models,
at substantially reduced computational cost.

\paragraph{LSUN Bedrooms}
\cref{tab:lsun-scores} presents scores for models trained on the LSUN Bedrooms dataset; samples from most of these models are shown in \cref{fig:lsun:samples}. Comparing the models' Inception scores with the one achieved by the test set makes clear that this measure is not meaningful for this dataset~--~not surprisingly, given the drastic difference in domain from ImageNet class labels.

In terms of KID and FID, MMD GANs outperform Cram\'er and WGAN-GP for each critic size. Although results with the smaller critic are worse than with the large one for each considered model,
small-critic MMD GANs still produce reasonably good samples, which certainly is not the case for WGAN-GP.
Although a small-critic Cram\'er GAN produces relatively good samples, the separate objects in these pictures often seem less sharp than the MMD \RQ{}* samples.
With a large critic, both Cram\'er GAN and MMD \RQ{}* give good quality samples, many of which are hardly distinguishable from the test set by eye.

\begin{table}[ht]
    \centering
 \caption{Mean (standard deviation) of score evaluations for the LSUN models. Inception scores do not seem meaningful for this dataset.}
    \label{tab:lsun-scores}
    \begin{tabular}{ccc|rrr}
      & \multicolumn{2}{c|}{critic size} & \multicolumn{3}{c}{} \\
 loss & filters & top layer & \multicolumn{1}{c}{Inception} & \multicolumn{1}{c}{FID} & \multicolumn{1}{c}{KID} \\
\hline
\RQ{}    &   16 &   16 &    3.13  (0.01) &   86.47  (0.29) &   0.091  (0.002)\\
\RQ{}    &   64 &   16 &    2.80  (0.01) &   31.95  (0.28) &   0.028  (0.002)\\
\DIST{}  &   16 &  256 &    3.42  (0.01) &  104.85  (0.32) &   0.109  (0.002)\\
\DIST{}  & 64   & 256  &    2.79  (0.01) &   35.28  (0.21) &   0.032  (0.001)\\
Cram\'er GAN&16 & 256  &    3.46  (0.02) &  122.03  (0.41) &   0.132  (0.002)\\
Cram\'er GAN&64 & 256  &    3.44  (0.02) &   54.18  (0.39) &   0.050  (0.002)\\
WGAN-GP  & 16   & 1    &    2.40  (0.01) &  292.77  (0.35) &   0.370  (0.003)\\
WGAN-GP  & 64   & 1    &    3.12  (0.01) &   41.39  (0.25) &   0.039  (0.002)\\
test set &  --  & --   &    2.36  (0.01) &    2.49  (0.02) &   0.000  (0.000)\\
    \end{tabular}
\end{table}

\begin{figure}[ht]
    \centering
    \begin{subfigure}[t]{0.30\textwidth}
        \centering
        \includegraphics[width=\linewidth]{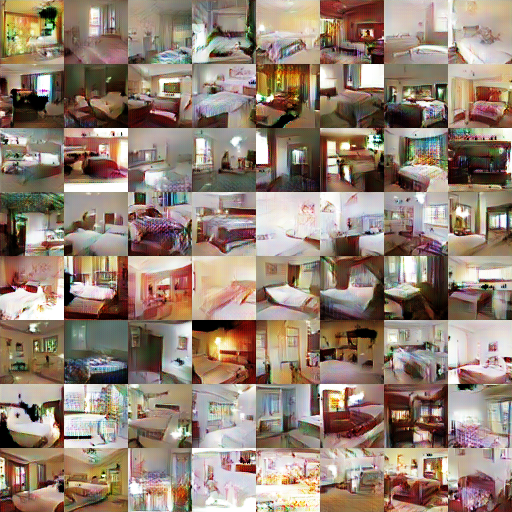}
        \caption{MMD \RQ{}, critic size 16} \label{fig:lsun:rq-16}
    \end{subfigure}
    \hfill
    \begin{subfigure}[t]{0.30\textwidth}
        \centering
        \includegraphics[width=\linewidth]{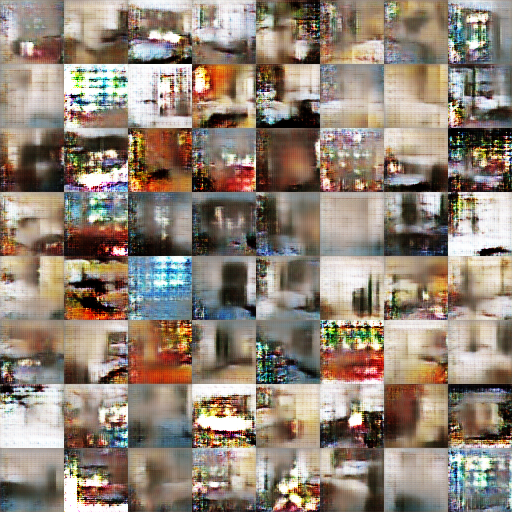}
        \caption{WGAN-GP, critic size 16} \label{fig:lsun:wgangp-16}
    \end{subfigure}
    \hfill
    \begin{subfigure}[t]{0.30\textwidth}
        \centering
        \includegraphics[width=\linewidth]{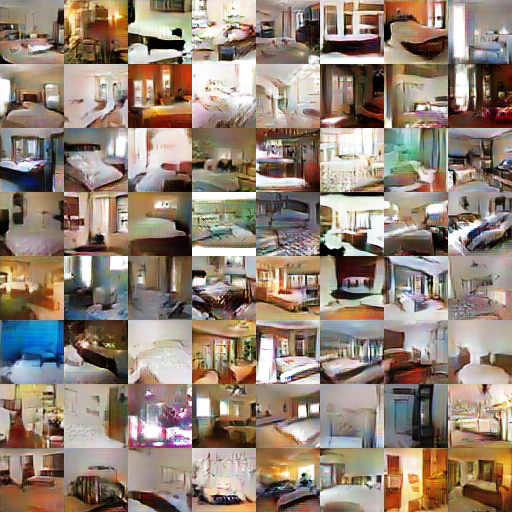}
        \caption{Cram\'er GAN, critic size 16} \label{fig:lsun:cramer-16}
    \end{subfigure}

    \vspace{0cm}

    \begin{subfigure}[t]{0.30\textwidth}
        \centering
        \includegraphics[width=\linewidth]{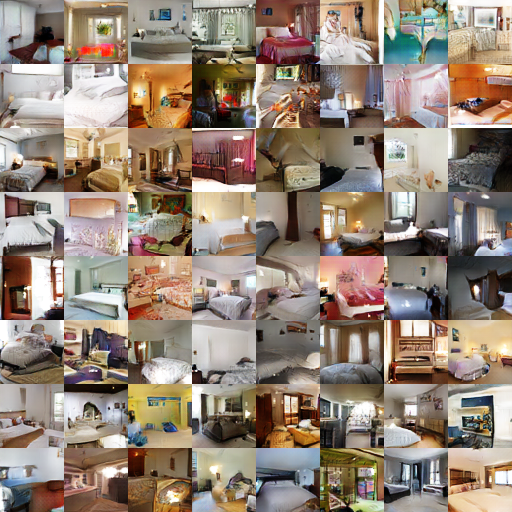}
        \caption{MMD \RQ{}, critic size 64} \label{fig:lsun:rq-64}
    \end{subfigure}
    \hfill
    \begin{subfigure}[t]{0.30\textwidth}
        \centering
        \includegraphics[width=\linewidth]{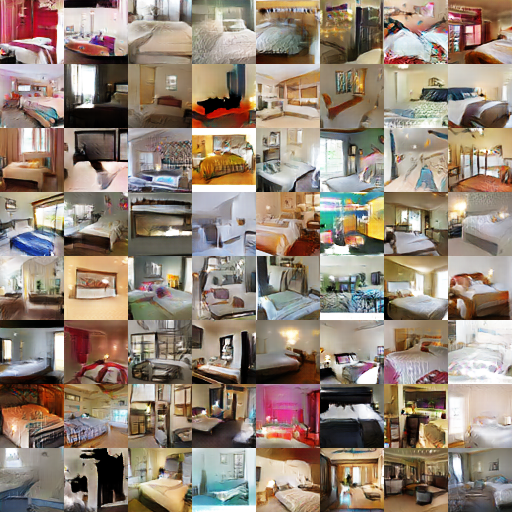}
        \caption{WGAN-GP, critic size 64} \label{fig:lsun:wgangp-64}
    \end{subfigure}
    \hfill
    \begin{subfigure}[t]{0.30\textwidth}
        \centering
        \includegraphics[width=\linewidth]{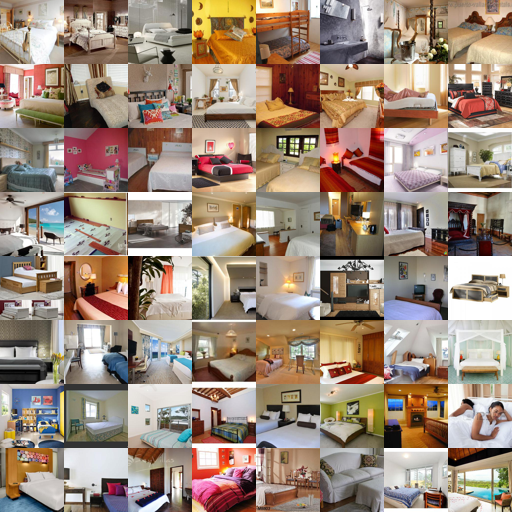}
        \caption{Test set} \label{fig:lsun:test}
    \end{subfigure}
    \caption{Comparison of samples for the $64 \times 64$ LSUN Bedroom database.}
    \label{fig:lsun:samples}
\end{figure}

\paragraph{CelebA}
Scores for the CelebA dataset are shown in \cref{tab:celeba-scores}; MMD GAN with \emph{rq*} kernel outperforms
both WGAN-GP and Cram\'er GAN in KID and FID. Samples in \cref{fig:celeba:samples} show that for each of the models there are
many visually pleasing pictures among the generated ones, yet unrealistic images are more common for WGAN-GP and Cram\'er.
\begin{table}[ht]
    \centering
 \caption{Mean (standard deviation) of score evaluations for the CelebA dataset.}
    \label{tab:celeba-scores}
    \begin{tabular}{ccc|rrr}
      & \multicolumn{2}{c|}{critic size} & \multicolumn{3}{c}{} \\
 loss & filters & top layer & \multicolumn{1}{c}{Inception} & \multicolumn{1}{c}{FID} & \multicolumn{1}{c}{KID} \\
\hline
\RQ{}    &   64 &   16 &    2.61  (0.01) &   20.55  (0.25) &   0.013  (0.001)\\
Cram\'er &   64 &   256 &    2.86  (0.01) &   31.30  (0.17) &   0.025  (0.001)\\
WGAN-GP  & 64   & 1    &    2.72  (0.01) &   29.24  (0.22) &   0.022  (0.001)\\
test set &  --  & --   &    3.76  (0.02) &    2.25  (0.04) &   0.000  (0.000)\\
    \end{tabular}
\end{table}

\begin{figure}[ht]
    \centering
    \begin{subfigure}[t]{0.30\textwidth}
        \centering
        \includegraphics[width=\linewidth]{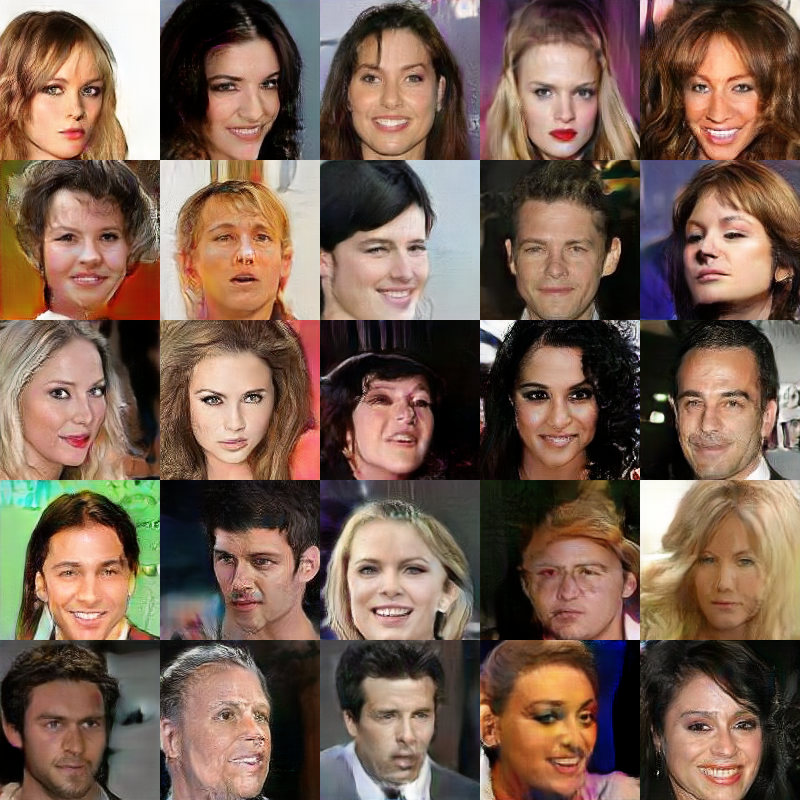}
        \caption{MMD \RQ{}*} \label{fig:celeba:rq}
    \end{subfigure}
    \hfill
    \begin{subfigure}[t]{0.30\textwidth}
        \centering
        \includegraphics[width=\linewidth]{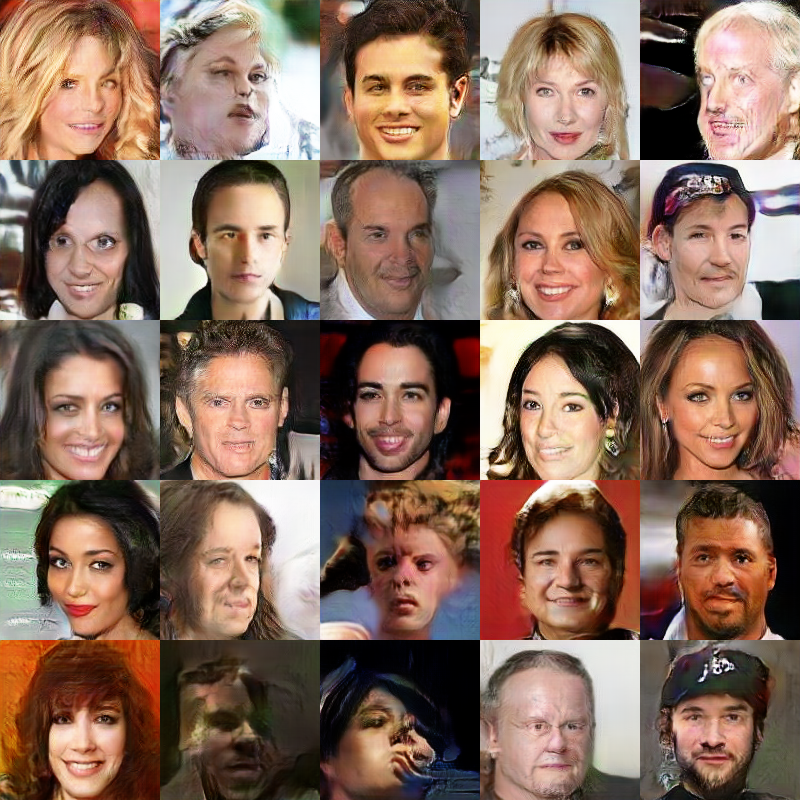}
        \caption{WGAN-GP} \label{fig:celeba:wgangp}
    \end{subfigure}
    \hfill
    \begin{subfigure}[t]{0.30\textwidth}
        \centering
        \includegraphics[width=\linewidth]{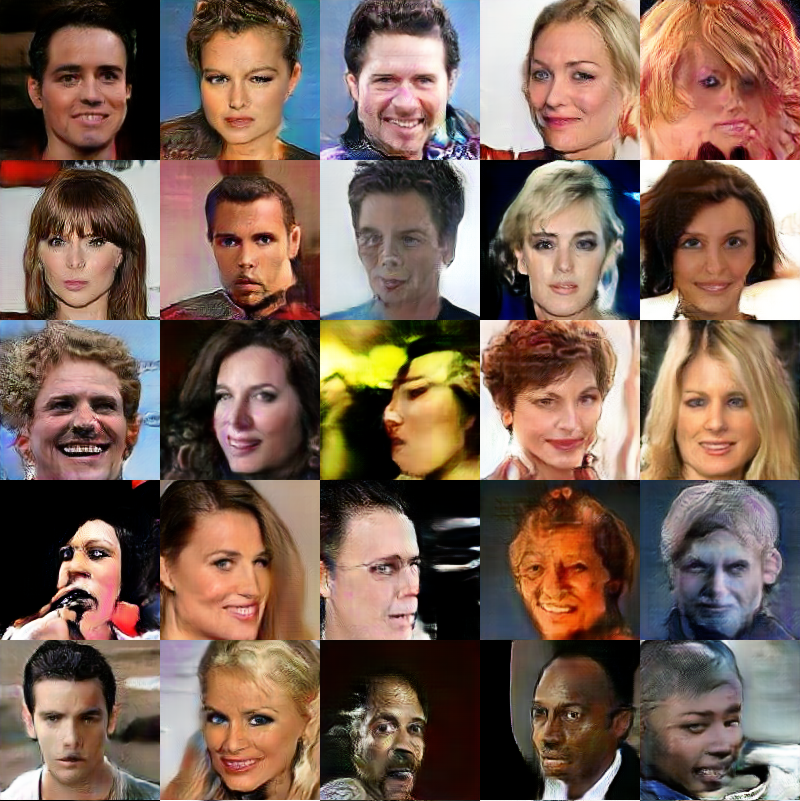}
        \caption{Cram\'er GAN} \label{fig:celeba:cramer}
    \end{subfigure}
    \caption{Comparison of samples with a ResNet generator for the $160 \times 160$ CelebA dataset.}
    \label{fig:celeba:samples}
\end{figure}

\paragraph{}
These results illustrate the benefits of using the MMD on deep convolutional feaures
as a GAN critic. In this hybrid system, the initial convolutional layers
map the generator and reference image distributions to a simpler representation, which
is well suited to comparison
via the MMD. The MMD in turn employs an infinite dimensional feature space
to compare the outputs of these convolutional layers. By comparison,  WGAN-GP requires a larger discriminator
network to achieve similar performance.
It is interesting to consider the question of kernel choice: the distance
kernel and RQ kernel are both characteristic \citep{SriGreFukLanetal10},
and neither suffers from the fast decay of the exponentiated
quadratic kernel, yet the RQ kernel performs slightly better in our experiments.
The relative merits of different kernel families for GAN training will
be an interesting topic for further study.

\bibliography{refs}
\bibliographystyle{iclr2018_conference}

\appendix

\section{Score functions, divergences, and the Cram\'er GAN} \label{sec:score-funcs}

It is not immediately obvious how to interpret the surrogate loss
(\ref{eq:surrogate}). An insight comes from considering the \emph{score
function} associated with the energy distance, which we now briefly
review \citep{GneRaf07}. A scoring rule is a function $S(\bP,y)$, which
is the loss incurred when a forecaster makes prediction $\bP$, and
the event $y$ is observed. The \emph{expected score}
is the expectation under $\bQ$ of the score,
\[
S(\bP,\bQ) := \E_{Y\sim \bQ} S(\bP, Y).
\]
If a score is \emph{proper}, then the expected score obtained when $\bP=\bQ$
is greater or equal than the expected score for $\bP \neq \bQ$,
\[
S(\bQ, \bQ) \ge S(\bP, \bQ)
.\]
A \emph{strictly proper} scoring rule shows an equality only when $\bP$ and
$\bQ$ agree. We can define a divergence measure based on this score,
\[
    \D_S(\bP,\bQ) = S(\bQ, \bQ) - S(\bP, \bQ)
    \label{eq:divergence}
.\]
Bearing in mind the definition of the divergence \eqref{eq:divergence},
it is easy to see \citep[eq. 22]{GneRaf07} that the energy distance \eqref{eq:energyDistance} arises from the score function
\[
S(\bP, y) = \frac{1}{2} \E_{\bP} \rho(X,X') - \E_{\bP} \rho(X, y).
\]
The interpretation is straightforward:
the score of a reference sample $y$
is determined by comparing its average distance to a generator sample
with the average distance among independent generator samples, $\E_\bP \rho(X, X')$.
If we take an expectation over $Y\sim \bQ$,
we recover the scoring rule optimized by the
DISCO Nets algorithm \cite[Section 3.3]{Bouchacourt:2016}.

As discussed earlier, the Cram\'er GAN critic does not use the energy distance
(\ref{eq:energyDistance}) directly on the samples, but first maps
the samples through a function $h$, for instance a convolutional
network; this should be chosen to maximize the discriminative performance
of the critic. Writing this  mapping as $h$, we break the energy
distance down as $\D_{e}(\bP,\bQ)=S(\bQ,\bQ)-S(\bP,\bQ)$, where
\[
S(\bQ,\bQ)=-\frac{1}{2} \E_{\bQ} \rho(h(Y),h(Y'))
\label{eq:ReferenceSimilarity}
\]
and
\[
S(\bP,\bQ) = \frac{1}{2} \E_{\bP} \rho(h(X),h(X')) - \E_{\bP,\bQ}\rho(h(X),h(Y))
\label{eq:expectedScore}
.\]
When training the discriminator, the goal is to maximize the divergence
by learning $h$, and so both (\ref{eq:ReferenceSimilarity}) and
(\ref{eq:expectedScore}) change: in other words, divergence maximization
is not possible without two independent samples $Y,Y'$ from the reference
distribution $\bQ.$

An alternative objective in light of the score interpretation, however,
is to simply optimize the average score (\ref{eq:expectedScore}).
In other words, we would find features $h$ that make the average distance from
generator to reference samples much larger than the average
distance between pairs of generator samples. We no longer control
the term encoding the ``variability'' due to $\bQ$, $\E_{\bQ} \rho(h(Y), h(Y'))$,
which might therefore explode:
for instance, $h$ might cause $h(Y)$ to disperse broadly,
and far from the support of $\bP$,
assuming sufficient flexibility to keep $\E_{\bP}\rho(h(X),h(X'))$ under control.
We can mitigate this by controlling the expected norm $\E_{\bQ} \rho(h(Y), 0)$,
which has the advantage of only requiring a single sample to compute.
For example, we could maximize
\[
    - \frac12 \E_\bP \rho(h(X), h(X'))
    + \E_{\bP,\bQ}\rho(h(X),h(Y))
    - \E_\bQ \rho(h(Y), 0)
    \label{eq:score-based-cramer}
.\]
This resembles the Cram\'er GAN critic \eqref{eq:surrogate},
but the generator-to-generator distance is scaled differently,
and there is an additional term:
$\E_{\bP} \rho(h(X), 0)$ is being maximized in \eqref{eq:surrogate}, which is more difficult to interpret.
An argument has been made (in personal communication with \citeauthor{cramer-gan}) that this last term is required if the function
$f_{c}$ in (\ref{eq:surrogateWitness}) is to be a witness of an
integral probability metric \eqref{eq:IPM},
although the asymmetry of this witness in $\bP$ vs $\bQ$ needs to be analyzed further.

\section{Bias of generalized IPM estimators} \label{appendix:dual-bias}

We will now show that all estimators of IPM-like distances and their gradients are biased.
\Cref{sec:general-ipms} defines a slight generalization of IPMs,
used to analyze MMD GANs in the same framework as WGANs,
and a class of estimators that are a natural model for the estimator used in GAN models.
\Cref{sec:gen-ipm-bias} both shows that not only are this form of estimators invariably biased in nontrivial cases, and moreover no unbiased estimator can possibly exist;
\cref{sec:gen-ipm-grad-bias} then demonstrates that any estimator with non-constant bias yields a biased gradient estimator.
\Cref{appendix:wgan-bias,appendix:max-bias} demonstrate specific examples of this bias for the Wasserstein and maximized-MMD distances.

\subsection{Generalized IPMs and data-splitting estimators} \label{sec:general-ipms}
We will first define a slight generalization of IPMs:
we will use this added generality to help analyze MMD GANs in \cref{appendix:max-bias}.
\begin{definition}[Generalized IPM] \label{def:generalized-ipm}
Let $\X$ be some domain, with $\M$ a class of probability measures on $\X$.\footnote{All results in this section could be trivially extended to support $\bP$ and $\bQ$ over different domains $\X$ and $\Y$, if desired.}
Let $\F$ be some parameter set,
and $J : \F \times \M \times \M \to \R$ an objective functional.
The \emph{generalized IPM} $\D : \M \times \M \to \R$ is then given by
\[
    \D(\bP, \bQ) = \sup_{f \in \F} J(f, \bP, \bQ)
    \label{eq:j-dist}
.\]
\end{definition}

For example,
if $\F$ is a class of functions $f : \X \to \R$, and $J$ is given by
\[
     J_{\IPM}(f, \bP, \bQ) = \E_{X \sim \bP} f(X) - \E_{Y \sim \bQ} f(Y)
     \label{eq:ipm-obj}
,\]
then we obtain integral probability metrics \eqref{eq:IPM}.
Given samples $\Xfull \sim \bP^m$ and $\Yfull \sim \bQ^n$,
let $\hat\bP$ denote the empirical distribution of $\Xfull$
(an equal mixture of point masses at each $X_i \in \Xfull$),
and similarly $\hat\bQ$ for $\Yfull$.
Then we have a simple estimator of \eqref{eq:ipm-obj} which is unbiased for fixed $f$:
\[
    \hat J_{\IPM}(f, \Xfull, \Yfull)
    = J_{\IPM}(f, \hat\bP, \hat\bQ)
    = \frac{1}{m} \sum_{i=1}^m f(X_i) - \frac1n \sum_{j=1}^n f(Y_j)
    \label{eq:ipm:obj-hat}
.\]

\begin{definition}[Data-splitting estimator of a generalized IPM] \label{def:data-splitting-est}
    Consider the distance \eqref{eq:j-dist},
    with objective $J$ and parameter class $\F$.
    Suppose we observe iid samples $\Xfull \sim \bP^m$, $\Yfull \sim \bQ^n$,
    for any two distributions $\bP$, $\bQ$.
    A \emph{data-splitting estimator} is a function $\hat\D(\Xfull, \Yfull)$
    which first randomly splits the sample $\Xfull$ into $\Xtr$, $\Xte$
    and $\Yfull$ into $\Ytr$, $\Yte$,
    somehow selects a critic function $\hat f_{\Xtr,\Ytr} \in \F$
    independently of $\Xte$, $\Yte$,
    and then returns a result of the form
    \[
         \hat\D(\Xfull, \Yfull)
       = \hat J(\hat f_{\Xtr,\Ytr}, \Xte, \Yte)
    ,\]
    where $\hat J(f, \X, \Y)$ is an estimator of $J(f, \bP, \bQ)$.
\end{definition}

These estimators are defined by three components:
the choice of relative sizes of the train-test split,
the selection procedure for $\hat f_{\Xtr,\Ytr}$,
and the estimator $\hat J$.
The most obvious selection procedure is
\[
    \hat f_{\Xtr,\Ytr} \in \argmax_{f \in \F} \hat J(f, \Xtr, \Ytr)
\label{eq:ipm:f-hat}
,\]
though of course one could use regularization or other techniques to select a different $f \in \F$,
and in practice one will use an approximate optimizer.
\citet{c2st} used an estimator of exactly this form in a two-sample testing setting.

As noted in \cref{sec:unbiased-gradients},
this training/test split is a reasonable match for the GAN training process.
As we optimize a WGAN-type model,
we compute the loss (or its gradients) on a minibatch,
while the current parameters of the critic are based only on data seen in previous iterations.
We can view the current minibatch as $\Xte, \Yte$,
all previously-seen data as $\Xtr, \Ytr$,
and the current critic function as $\hat f_{\Xtr,\Ytr}$.
Thus, at least in the first pass over the training set,
WGAN-type approaches exactly fit the data-splitting form of \cref{def:data-splitting-est};
in later passes, the difference from this setup should be relatively small unless the model is substantially overfitting.

\subsection{Estimator bias} \label{sec:gen-ipm-bias}
We first show, in \cref{thm:ipm-bias},
that data-splitting estimators are biased downwards.
Although this provides substantial intuition about the situation in GANs,
it leaves open the question of whether some other unbiased estimator might exist;
\cref{thm:ipm-always-biased} shows that this is not the case.

\begin{theorem} \label{thm:ipm-bias}
    Consider a data-splitting estimator (\cref{def:data-splitting-est}) of the generalized IPM $\D$ (\cref{def:generalized-ipm})
    based on an unbiased estimator $\hat J$ of $J$:
    for any fixed $f \in \F$,
    \[
      \E_{\substack{\Xfull \sim \bP^m\\\Yfull \sim \bQ^n}}\left[ \hat J(f, \Xfull, \Yfull) \right]
      = J(f, \bP, \bQ)
    .\]

    Then either the selection procedure is almost surely perfect,
    \[
        \Pr\left( J\left( \hat f_{\Xtr,\Ytr}, \bP, \bQ \right) = \D(\bP, \bQ) \right) = 1
        \label{eq:j-est:train-perfectly}
    ,\]
    or else the estimator has a downward bias:
    \[
        \E \hat\D\left(\Xfull, \Yfull \right) < \D\left( \bP, \bQ \right)
        \label{eq:j-est:biased}
    .\]
\end{theorem}
\begin{proof}
Since $\Xtr, \Ytr$ are independent of $\Xte$, $\Yte$,
\[
     \E \hat\D(\Xfull, \Yfull)
   = \E_{\substack{\Xtr,\Ytr\\\Xte,\Yte}}\left[
        \hat J(\hat f_{\Xtr,\Ytr}, \Xte, \Yte)
     \right]
   = \E_{\Xtr, \Ytr}\left[
        J(\hat f_{\Xtr,\Ytr}, \bP, \bQ)
     \right]
\label{eq:ipm:exp-hat}
.\]
Define the suboptimality of $\hat f_{\Xtr, \Ytr}$ as
\[
  \varepsilon := \D(\bP, \bQ) - J\left( \hat f_{\Xtr,\Ytr}, \bP, \bQ \right)
,\]
so that $\E \hat\D(\Xfull, \Yfull) = \D(\bP, \bQ) - \E[\varepsilon]$.
Note that $\varepsilon \ge 0$,
since $\D(\bP, \bQ) = \sup_{f \in \F} J(f, \bP, \bQ)$
and so for any $f \in \F$ we have
\[
    J(f, \bP, \bQ) \le \D(\bP, \bQ)
.\]

Thus,
either $\Pr(\varepsilon = 0) = 1$, in which case \eqref{eq:j-est:train-perfectly} holds,
or else $\E[\varepsilon] > 0$, giving \eqref{eq:j-est:biased}.
\end{proof}

\Cref{thm:ipm-bias} makes clear that as $\hat f_{\Xtr,\Ytr}$ converges to its optimum,
the bias of $\hat\D$ should vanish \citep[as in][Theorem 3]{cramer-gan}.
Moreover, in the GAN setting the minibatch size only directly determines $\Xte, \Yte$,
which do not contribute to this bias;
bias is due rather to the training procedure and the number of samples seen through the training process.
As long as $\hat f_{\Xtr,\Ytr}$ is not optimal, however, the estimator will remain biased.

Many estimators of IPMs do not actually perform this data splitting procedure,
instead estimating $\D(\bP, \bQ)$
with the distance between empirical distributions $\D(\hat\bP, \hat\bQ)$.
The standard biased estimator of the MMD \citep[Equation 5]{mmd-jmlr},
the IPM estimators of \citet{ipm-empirical-est},
and the empirical Wasserstein estimator studied by \citet{cramer-gan}
are all of this form.
These estimators, as well as any other conceivable estimator,
are also biased:
\begin{theorem} \label{thm:ipm-always-biased}
  Let $\mathcal P$ be a class of distributions
  such that
  $\{ (1 - \alpha) \PP_0 + \alpha \PP_1 : 0 \le \alpha \le 1 \} \subseteq \mathcal P$,
  where $\PP_0 \ne \PP_1$ are two fixed distributions.
  Let $\D$ be an IPM \eqref{eq:IPM}.
  There does not exist any estimator of $\D$ which is unbiased on $\mathcal P$.
\end{theorem}
\begin{proof}
  We use a technique inspired by \citet{unbiased-convex}.
  Suppose there is an unbiased estimator $\hat\D(\Xset, \Yset)$ of $\D$:
  for some finite $m$ and $n$,
  if $\Xset = \{X_1, \dots, X_m\} \sim \PP^m$, $\Yset \sim \QQ^n$,
  then $\E[ \hat\D(\Xset, \Yset) ] = \D(\PP, \QQ)$.

  Fix $\PP_0$, $\PP_1$, and $\QQ \in \mathcal P$,
  and consider the function
  \begin{align}
     R(\alpha)
  &= \D((1 - \alpha) \PP_0 + \alpha \PP_1, \QQ)
\\&= \int_{X_1} \!\cdots\! \int_{X_m} \! \int_{\Yset} \hat\D(\Xset, \Yset)
     \,\ud\left[ (1 - \alpha) \PP_0 + \alpha \PP_1 \right](X_1)
     \cdots
     \ud\left[ (1 - \alpha) \PP_0 + \alpha \PP_1 \right](X_m)
     \,\ud\QQ(\Yset)
\\&= \int_{X_1} \!\cdots\! \int_{X_m} \! \int_{\Yset} \hat D(\Xset, \Yset)
     \left[ (1 - \alpha) \,\ud\PP_0(X_1) + \alpha \,\ud\PP_1(X_1) \right]
     \cdots
     \,\ud\QQ(\Yset)
\\&= (1 - \alpha)^m \E_{\substack{\Xset \sim \PP_0^m \\ \Yset \sim \QQ^n}}\left[ \hat\D(\Xset, \Yset) \right]
   + \dots
   + \alpha^m \E_{\substack{\Xset \sim \PP_1^m \\ \Yset \sim \QQ^n}}\left[ \hat\D(\Xset, \Yset) \right]
  .\end{align}
  Thus $R(\alpha)$ is a polynomial in $\alpha$ of degree at most $m$.

  But taking $\QQ = \frac12 \PP_0 + \frac12 \PP_1$ gives
  \begin{align}
    R(\alpha)
  &= \D\left((1-\alpha) \PP_0 + \alpha \PP_1, \tfrac12 \PP_0 + \tfrac12 \PP_1 \right)
\\&= \sup_{f \in \F} \E_{(1-\alpha) \PP_0 + \alpha \PP_1} f(X)
                   - \E_{\tfrac12  \PP_0 + \tfrac12 \PP_1} f(Y)
\\&= \sup_{f \in \F} (1-\alpha) \E_{\PP_0} f(X)
                    + \alpha \E_{\PP_1} f(X)
                    - \tfrac12 \E_{\PP_0} f(X)
                    - \tfrac12 \E_{\PP_1} f(X)
\\&= \sup_{f \in \F} \left( \tfrac12 - \alpha \right) \E_{\PP_0} f(X)
                   - \left( \tfrac12 - \alpha \right) \E_{\PP_1} f(X)
\\&= \left\lvert \tfrac12 - \alpha \right\rvert \sup_{f \in \F} \E_{\PP_0} f(X) - \E_{\PP_1} f(X)
\label{eq:R-alpha-abs}
\\&= \left\lvert \tfrac12 - \alpha \right\rvert \D(\PP_0, \PP_1)
\label{eq:R-alpha-ipm}
  ,\end{align}
  where \eqref{eq:R-alpha-abs} used our general assumption about IPMs that if $f \in \F$, we also have $-f \in \F$.
  But $R(\alpha)$ is not a polynomial with any finite degree.
  Thus no such unbiased estimator $\hat\D$ exists.
\end{proof}

Note that the proof of \cref{thm:ipm-always-biased}
does not readily extend to generalized IPMs,
and so does not tell us whether an unbiased estimator of the MMD GAN objective \eqref{eq:sup-mmd} can exist.
Also, attempting to apply the same argument to squared IPMs
would give the square of \eqref{eq:R-alpha-ipm},
which is a quadratic function in $\alpha$.
Thus tells us that although no unbiased estimator for a squared IPM can exist with only $m = 1$ sample point,
one can exist for $m \ge 2$,
as indeed \eqref{eq:unbiasedMMD} does for the squared MMD.

\subsection{Gradient estimator bias} \label{sec:gen-ipm-grad-bias}
We will now show that biased estimators, except for estimators with a constant bias,
must also have biased gradients.

Assume that, as in the GAN setting,
$\Q$ is given by a generator network $G_{\psi}$
with parameter $\psi$ and inputs $Z \sim \Z$,
so that $Y = G_{\psi}(Z) \sim \Q_\psi$.
The generalized IPM of \eqref{eq:j-dist} is now a function of $\psi$, which we will denote as
\[
  \D(\psi) := \D(\PP, \Q_{\psi})
.\]

Consider an estimator $\hat{\D}(\psi)$ of $\D(\psi)$.
\Cref{thm:biased-means-biased-grad} shows that when $\hat{\D}(\psi)$ and $\D(\psi)$ are differentiable,
the gradient $\nabla_{\psi}\hat{\D}(\psi)$ is an unbiased estimator for $\nabla_{\psi}\D(\psi)$  only if the bias of $\hat{\D}(\psi)$ doesn't depend on $\psi$.
This is exceedingly unlikely to happen for the biased estimator $\hat{\D}(W)$ defined in \cref{thm:ipm-bias},
and indeed \cref{thm:ipm-always-biased} shows cannot happen for any IPM estimator.

\begin{theorem} \label{thm:biased-means-biased-grad}
    Let $\D : \Psi \to \R$ be a function on a parameter space $\Psi \subseteq \R^d$,
    with a random estimator $\hat\D : \Psi \to \R$ which is almost surely differentiable.
    Suppose that $\hat\D$ has unbiased gradients:
    \[
        \E[\nabla_{\psi}\hat{\D}(\psi)] = \nabla_{\psi}\D(\psi)
    .\]
    Then, for each connected component of $\Psi$,
    \[
        \E \hat\D(\psi) = D(\psi) + \mathrm{const}
    ,\]
    where the constant can vary only across distinct connected components.
\end{theorem}
\begin{proof}
    Let $\psi_1$ and $\psi_2$ be an arbitrary pair of parameter values in $\Psi$,
    connected by some smooth path $r : [0, 1] \to \Psi$
    with $r(0) = \psi_1$, $r(1) = \psi_2$.
    For example, if $\Psi$ is convex, then paths of the form $r(t) = t \psi_1 + (1-t) \psi_2$ are sufficient.
    Using Fubini's theorem and standard results about path integrals, we have that
    \begin{align}
        \E[ \hat{\D}(\psi_1) - \hat{\D}(\psi_2)  ]
        &= \E\left[ \int_{0}^1  \left( \nabla \hat{\D}( r(t) ) \right) \cdot r'(t) \,\ud t \right] \\
        &= \int_0^1 \E\left[ \nabla \hat{\D}( r(t) ) \right] \cdot r'(t) \,\ud t \\
        &= \int_0^1 \left( \nabla \D( r(t) ) \right) \cdot r'(t) \,\ud t \\
        &= \D(\psi_1) - \D(\psi_2)
    .\end{align}
This implies that $\E[\hat{\D}(\psi)] = D(\psi) + \mathrm{const}$
for all $\psi$ in the same connected component of $\Psi$.
\end{proof}

\subsection{WGANs} \label{appendix:wgan-bias}

\Cref{thm:ipm-bias,thm:ipm-always-biased} hold for the original WGANs,
whose critic functions are exactly $L$-Lipschitz,
considering $\F$ as the set of $L$-Lipschitz functions
so that $\D_\F$ is $L$ times the Wasserstein distance.
They also hold for either WGANs or WGAN-GPs
with $\F$ the actual set of functions attainable by the critic architecture,
so that $\D$ is the ``neural network distance'' of \citet{arora:gen-equilibrium}
or the ``adversarial divergence'' of \citet{approx-convergence-props}.

It should be obvious that for nontrivial distributions $\bP$ and $\bQ$
and reasonable selection criteria for $\hat f_{\Xtr, \Ytr}$,
\eqref{eq:j-est:train-perfectly} does not hold,
and thus \eqref{eq:j-est:biased} does (so that the estimate is biased downwards).
\Cref{thm:ipm-always-biased} also shows this is the case on reasonable families of input distributions,
and moreover that the bias is not constant,
so that gradients are biased by \cref{thm:biased-means-biased-grad}.

\paragraph{Example}
For an explicit demonstration,
consider the Wasserstein case,
$\F$ the set of $1$-Lipschitz functions,
with $\bP = \N(1, 1)$ and $\bQ = \N(0, 1)$.
Here $\D_\F(\bP, \bQ) = 1$;
the only critic functions $f \in \F$ which achieve this are $f(t) = t + C$ for $C \in \R$.

If we observe only one training pair $\xtr \sim \bP$ and $\ytr \sim \bQ$,
when $\xtr > \ytr$,
$f_1(t) = t$ is a maximizer of \eqref{eq:ipm:f-hat},
leading to the expected estimate $J(f_1, \bP, \bQ) = 1$.
But with probability $\Phi\left(-1 / \sqrt{2}\right) \approx 0.24$
it happens that
$\xtr < \ytr$.
In such cases,
\eqref{eq:ipm:f-hat} could give e.g.\ $f_{-1}(t) = - t$,
giving the expected response $J(f_{-1}, \bP, \bQ) = -1$;
the overall expected estimate of the estimator using this critic selection procedure is then
$\E \hat\D_\F(\Xfull, \Yfull) = \left( 1 - \Phi\left(-1 / \sqrt 2\right) \right) - \Phi\left(-1 / \sqrt 2 \right) \approx 0.52$.

The only way to achieve
$\E \hat\D_\F(\Xfull, \Yfull) = 1$
would be a ``stubborn'' selection procedure
which chooses $f_1 + C$ no matter the given inputs.
This would have the correct output $\E \hat\D_\F(\Xfull, \Yfull) = 1$ for this $(\bP, \bQ)$ pair.
Applying this same procedure to $\bP = \N(-1, 1)$ and $\bQ = \N(0, 1)$,
however,
would then give $\E \hat\D_\F(\Xfull, \Yfull) = -1$, when it should also be $1$.

\subsection{Maximal MMD estimator} \label{appendix:max-bias}
Recall the distance $\eta(\bP, \bQ) = \sup_\theta \MMD^2(h_\theta(\bP), h_\theta(\bQ))$
defined by \eqref{eq:sup-mmd}.
MMD GANs can be viewed as estimating $\eta$ according to the scheme of \cref{thm:ipm-bias},
with $\F$ the set of possible parameters $\theta$,
$J(\theta, \bP, \bQ) = \MMD^2(h_\theta(\bP), h_\theta(\bQ))$,
and $\hat J(\theta, \Xfull, \Yfull) = \MMD_u^2(h_\theta(\Xfull), h_\theta(\Yfull))$.
Clearly our optimization scheme for $\theta$ does not almost surely yield perfect answers,
and so again we have
\[
    \E \hat\eta(\Xfull, \Yfull) < \eta(\bP, \bQ)
.\]

As $\mtr, \ntr \to \infty$,
as for Wasserstein it should be the case that
$\hat\eta \to \eta$.
This is shown for certain kernels, along with the rate of convergence,
by \citet[Section 4]{kernel-choice-mmd}.

It should also be clear that in nontrivial situations,
this bias is not constant,
and hence gradients are biased by \cref{thm:biased-means-biased-grad}.

\paragraph{Example}
For a particular demonstration, consider
\[
    \bP = \N\left(\begin{bmatrix}1 \\ 0\end{bmatrix}, I\right)
    ,\qquad
    \bQ = \N\left(\begin{bmatrix}0 \\ 0\end{bmatrix}, I\right)
,\]
with $h_\theta : \R^2 \to \R$ given by $h_\theta(x) = \theta\tp x$, $\lVert \theta \rVert = 1$,
so that $h_\theta$ chooses a one-dimensional projection of the two-dimensional data.
Then use the linear kernel $k^\DOT$,
so that the MMD is simply the difference in means between projected features:
\[
    \MMD^2(h_\theta(\bP), h_\theta(\bQ))
    = \lVert \E \theta\tp X - \E \theta\tp Y \rVert^2
    = \theta_1^2
,\]
and
\[
     \E \hat\eta(\Xfull, \Yfull)
   = \E_{\Xtr,\Ytr}\left[ \left( \hat\theta_{\Xtr,\Ytr} \right)_1^2 \right]
.\]

Clearly $\eta(\bP, \bQ) = 1$,
which is obtained by $\theta \in \{ (-1, 0), (1, 0) \}$;
any other valid $\theta$ will yield a strictly smaller value of $\MMD^2(h_\theta(\bP), h_\theta(\bQ))$.

The MMD GAN estimator of $\eta$, if the optimum is achieved, uses
\begin{align}
    \hat\theta_{\Xtr,\Ytr}
  &= \argmax_{\theta : \lVert \theta \rVert = 1} \MMD_u^2(\Xtr, \Ytr)
\\&= \argmax_{\theta : \lVert \theta \rVert = 1}
        \frac{1}{\mtr (\mtr - 1)} \sum_{i \ne j}^{\mtr} (\theta\tp \xtr_i) ({\xtr_j}\tp \theta)
      + \frac{1}{\ntr (\ntr - 1)} \sum_{i \ne j}^{\ntr} (\theta\tp \ytr_i) ({\ytr_j}\tp \theta)
\\&\phantom{\argmax_{\theta : \lVert \theta \rVert = 1}}\qquad
      - \frac{2}{\mtr \ntr} \sum_{i=1}^{\mtr} \sum_{j=1}^{\ntr} (\theta\tp \xtr_i) ({\ytr_j}\tp \theta)
\\&= \argmax_{\theta : \lVert \theta \rVert = 1} \theta\tp A_{\Xtr,\Ytr} \theta
,\end{align}
where $A_{\Xtr,\Ytr} \in \R^{2 \times 2}$.
$\hat\theta_{\Xtr,\Ytr}$ is then the normalized first eigenvector of $A_{\Xtr,\Ytr}$.
For Gaussian $\Xtr$, $\Ytr$ with finite $\mtr$, $\ntr$,
$\hat\theta_{\Xtr,\Ytr}$ is a continuous random variable
and so does not almost surely lie in $\{(-1, 0), (1, 0)\}$;
thus by \cref{thm:ipm-bias},
$\E \hat\eta(\Xfull, \Yfull) < \eta(\bP, \bQ) = 1$.
(A numerical simulation for the former gives a value around $0.6$ when $\mtr = \ntr = 2$.)

\section{Proof of unbiased gradients} \label{appendix:proof}

We now proceed to prove \cref{thm:main-body}
as a corollary to the \cref{thm:interversion}, our main result about exchanging gradients and expectations of deep networks.

Exchanging the gradient and the expectation can often be guaranteed using a standard result in measure theory (see \cref{thm:exchanging_gradient_expectation}), as a corollary of the Dominated Convergence theorem (\cref{thm:DC}).
This result, however, requires the property \cref{thm:exchanging_gradient_expectation}.\ref{diff}:
for almost all inputs $X$, the mapping is differentiable on the entirety of a neighborhood around $\theta$.
This order of quantifiers is important:
it allows the use of the mean value theorem to control the average rate of change of the function,
and the result then follows from \cref{thm:DC}.

For a neural network with the ReLU activation function, however, this assumption doesn't hold in general.
For instance, if $\theta = (\theta_1, \theta_2)\in \R^2$ with $\theta_2 \neq 0$ and $X\in \R$,
one can consider this very simple function:
$h_{\theta}(X) = \max(0, \theta_1 + \theta_2 X)$.
For any fixed value of $\theta$, the function $h_{\theta}(X)$ is differentiable in $\theta$
for all $X $ in  $\R$ except for $X_{\theta} = - \theta_1 / \theta_2$.
However, if we consider a ball of possible $\theta$ values $B(\theta, r)$,
the function is not differentiable on the set
$\left\{ - \theta_1' / \theta_2'  \in \R \mid \theta' \in B(\theta, r)  \right\}$,
which can have positive measure for many possible distributions for $X$.

In \cref{thm:interversion},
we provide a proof that derivatives and expectations can be exchanged for all parameter values outside of a ``bad set'' $\Theta_\PP$,
without relying on \cref{thm:exchanging_gradient_expectation}.\ref{diff}.
This can be done using \cref{network_growth},
which takes advantage of the particular structure of neural networks to control the average rate of change without using the mean value theorem.
Dominated convergence (\cref{thm:DC}) can then be applied directly.

We also show in \cref{main_zero_measure} that the set $\Theta_\PP$,
of parameter values where \cref{thm:interversion} might not hold,
has zero Lebesgue measure.
This relies on the standard Fubini theorem \citep[Theorem 14.16]{Klenke:2008}
and \cref{zero_measure},
which ensures that the network $\theta \mapsto h_{\theta}(X)$ is differentiable for almost all parameter values $\theta$ when $X$ is fixed.
Although \cref{zero_measure} might at first sight seem obvious,
it requires some technical considerations in topology and differential geometry.

\begin{prop}[{{Differentiation Lemma \citep[e.g.][Theorem 6.28]{Klenke:2008}}}] \label{thm:exchanging_gradient_expectation}
    Let $V$ be a nontrivial open set in $\R^m$ and let $\PP$ be a probability distribution on $\R^d $. Define a map $h: \R^d \times V \mapsto \R^n $ with the following properties:
    \begin{enumerate}[label=(\roman*)]
        \item For any $\theta \in V $, $\E_{\PP}[ \lVert h_{\theta}(X) \rVert ] < \infty$.
        \item \label{diff} For $\PP$-almost all $X \in \R^d$, the map $V \to \R^n$, $\theta \mapsto h_{\theta}(X) $ is differentiable.
        \item There exists a $\PP$-integrable function $g : \R^d\mapsto \R $ such that $\Vert  \partial_{\theta} h_{\theta}(X)\Vert \leq g(X) $ for all $\theta\in V$.
    \end{enumerate}
    Then, for any $\theta \in V$, $ \E_{\PP}[ \lVert \partial_{\theta} h_{\theta}(X) \rVert ]< \infty$ and the function $  \theta \mapsto \E_{\PP}[ h_{\theta}(X) ] $ is differentiable with differential:
        \[
        \partial_{\theta} \E_{\PP}[  h_{\theta}(X) ] =  \E_{\PP}[ \partial_{\theta} h_{\theta}(X) ]
       .\]
\end{prop}

\begin{prop}[{{Dominated Convergence Theorem \citep[e.g.][Corollary 6.26]{Klenke:2008}}}]\label{thm:DC}
    Let $\PP$ be a probability distribution on $\R^d$ and $f$ a measurable function.
    Let $(f_n)_{n\in \n}$ be a sequence of of integrable functions such that for $\PP$-almost all $X \in \R^d$, $f_n(X) \to f(X) $ as $n$ goes to $\infty$.
    Assume that there is a dominating function $g$: $\lVert f_n(X) \rVert \le g(X)$ for $\PP$-almost all $X\in \R^d$ for all $n\in \n$, and $\E_{\PP}[g(X)] < \infty$.
    Then $f$ is $\PP$-integrable, and $ \E_{\PP}[f_n(X)] \rightarrow \E_{\PP}[f(X)]   $ as $n$ goes to $\infty$.
\end{prop}

\subsection{Network definition} \label{sec:proof:notation}

We would like to consider general feed-forward networks with a directed acyclic computation graph $G$.
Here, $G$ consists of $L+1$ nodes, with a root node $i=0 $ and a leaf node $i=L$. We denote by $\pi(i) $ the set of parent nodes of $i$. The nodes are sorted according to a topological order: if $j$ is a parent node of $i$, then $j < i$.
Each node $i$ for $i > 0$ computes a function $f_i$,
which outputs a vector in $\R^{d_i}$
based on its input in $\R^{d_{\pi(i)}}$,
the concatenation of the outputs of each layer in $\pi(i)$.
Here $d_{\pi(i)} = \sum_{j \in \pi(i)} d_j$, and $d_0 = d > 0$.

We define the feed-forward network that factorizes according the graph $G$ and with functions $f_i$ recursively:
\begin{align}
h^{0} &= X\\
h^{i} &= f_i(h^{\pi(i)}) \qquad \forall\; 0 < i \leq L
,\end{align}
where $h^{\pi(i)}$ is the concatenation of the vectors $h^{j}$ for $j\in \pi(i)$.
The functions $f_i$ can be of two types:
\begin{itemize}
    \item \textit{Affine transform (Linear Module)}:
    \begin{align}
        f_i(Y) =  g_i(W^{i}) \begin{bmatrix}Y \\ 1\end{bmatrix}
    \end{align}
    $W^{i}$ is a vector of dimension $m_i$.
    The function $g_i : \R^{m_i} \to \R^{d_i \times (d_{\pi(i)} + 1)}$ is a known linear operator on the weights $W^i$, which can account for convolutions and similar linear operations.
    We will sometimes use $\widetilde{Y}$ to denote the augmented vector  $\begin{bmatrix}Y \\ 1\end{bmatrix}$, which accounts for bias terms.
    \item \textit{Non-linear}: These $f_i$ have no learnable weights.
    $f_i$ can potentially be non-differentiable, such as max pooling, ReLU, and so on.
    Some conditions on $f_i$ will be required (see \cref{Diff_set}); the usual functions used in practice satisfy these conditions.
    \end{itemize}

Denote by $C$ the set of nodes $i$ such that $f_i$ is non-linear.  $\theta$ is the concatenation of parameters of all linear modules: $\theta = (W^{k})_{ k\in C^{c}  }$, where $C^{c}$ is the complement of $C$ in $[L] = \{1,...,L\}$.
Call the total number of parameters $m = \sum_{i \in C^c} m_i$, so that $\theta \in \R^{m}$.
The feature vector of the network corresponds to the output of the last node $L$ and will be denoted $h_{\theta} := h_{\theta}^{L} \in \R^{d_L}$. The subscript $\theta$ stands for the parameters of the network.
We will sometimes use $h_\theta(X)$ to denote explicit dependence on $X$,
or omit it when $X$ is fixed.

Also define a ``top-level function'' to be applied to $h_\theta$, $\K : \R^{d_L}\to \R$.
This function might simply be $\K(U) = U$, as in \cref{cor:wgan,cor:gan}.
But it also allows us to represent the kernel function of an MMD GAN in \cref{cor:mmd-gan}:
here we take $X$ to be the two inputs to the kernel stacked together,
apply the network to each of the two inputs with the same parameters in parallel,
and then compute the kernel value between the two representations with $\K$.
$\K$ will have different smoothness assumptions than the preceding layers (\cref{kernel_growth}).

\subsection{Assumptions} \label{sec:proof:assumptions}
We will need the following assumptions at various points,
where $\alpha \ge 1$:
\begin{assumplist}
\item \label{moments} (Moments) $\E_{\PP}[\lVert X \rVert^\alpha ]< \infty$.
\item \label{kernel_growth} The function $\K$ is continuously differentiable, and satisfies the following growth conditions where $C_0$ and $C_1$ are constants:
 \begin{align}
    \lvert \K(U) \rvert &\leq C_0 ( \Vert U\Vert^{\alpha}    +1) \\
    \Vert \nabla  \K(U) \Vert &\leq  C_1 ( \Vert U\Vert^{\alpha-1}  + 1)
.\end{align}
\item \label{Lipschitz} (Lipschitz nonlinear layers) For each $i \in C$, $f_i$ is $M$-Lipschitz.
\item \label{Diff_set} (Analytic pieces)
    For each nonlinear layer $f_i$, $i \in C$,
    there are $K_i$ functions $(f_i^{k})_{k \in [K_i]}$,
    each real analytic on $\R^{d_{\pi(i)}}$,
    which agree with $f_i$ on the closure of a set $\D_i^k$:
    \[
        f_i(Y) = f_i^{k}(Y) \qquad \forall Y\in \bar{\mathcal{D}}_{i}^{k}
    .\]
    These sets $\D_i^k$ are disjoint,
    and cover the whole input space: $\bigcup_{k=1}^{K_i} \bar\D_i^k = \R^{d_{\pi(i)}}$.
    Moreover, each $\D_i^k$ is defined by $S_{i,k}$ real analytic functions
    $G_{i,k,s} : \R^{d_{\pi(i)}} \to \R$ as
    \[
        \D_i^k = \left\{
            Y \in \R^{d_{\pi(i)}}
            \mid
            \forall s \in [S_{i,k}],
            G_{i,k,s}(Y) > 0
        \right\}
    .\]
\end{assumplist}

Note that \cref{kernel_growth} is satisfied by the function $K(U) = U$, used in \cref{cor:wgan,cor:gan}, with $\alpha = 1$.
It is also satisfied by the top-level functions of an MMD GAN with each of the kernels we consider in this work; see \cref{cor:mmd-gan}.

\Cref{Lipschitz,Diff_set}
are satisfied by the vast majority of deep networks used in practice.

For example, if $f_i$ computes the ReLU activation function on two inputs,
then we have $K_i = 4$,
with each $\D_i^k$ corresponding to a quadrant of the real plane (see \cref{RELU_domains}).
These quadrants might each be defined by $S_{i,k} = 2$ inequalities of the form $G_{i,k,1}(Y)>0$ and $G_{i,k,2}(Y)>0$, where $G_{i,k,s}(Y) = \pm Y_1$ and $G_{i,k,s}(Y) = \pm Y_2$ are analytic. Moreover, on each of these domains $f_i$ coincides with an analytic function:
\begin{align}
    f_i(Y) = \left\{\begin{array}{ll}
        (Y_1,Y_2); & Y\in \D_i^1\\
        (Y_1,0); & Y\in \D_i^2\\
        (0,0); & Y\in \D_i^3\\
        (0,Y_2); & Y\in \D_i^4
    \end{array}\right.
.\end{align}

Another example is when $f_i$ computes max-pooling on two inputs. In that case we have $K_i = 2$, and each domain $\D_i^k$ corresponds to a half plane (see \cref{Max_domains} ). Each domain is defined by one inequality $ G_{i,k,1}(Y) > 0 $   with $G_{i,1,1}(Y) =  Y_1-Y_2  $ and $G_{i,2,1}(Y) =  Y_2-Y_1$. Again, $G_{i,k,1}$ are analytic functions and $f_i$ coincides with an analytic function on each of the domains:

\begin{align}
    f_{i}(Y) = \left\{\begin{array}{ll}
        Y_1; & Y\in \D_i^1\\
        Y_2; & Y\in \D_i^2\\
    \end{array}\right.
.\end{align}

When $f_i$ is analytic on the whole space,
$\D_i = \R^{d_{\pi(i)}}$,
we can choose $K_i = 1$ to get $\D_i^1 = \R^{d_{\pi(i)}}$,
which can be defined by a single function ($S_{i,k} = 1$) of $G_{i,1,1}(Y) = 1$.
This case corresponds to most of the differentiable functions used in deep learning,
such as the softmax, sigmoid, hyperbolic tangent, and batch normalization functions.

Other activation functions, such as the ELU \citep{elu}, are piecewise-analytic and also satisfy \cref{Lipschitz,Diff_set}.

\begin{figure*}[t!]
    \begin{subfigure}[t]{0.5\linewidth}
        \includegraphics[height=0.7\linewidth]{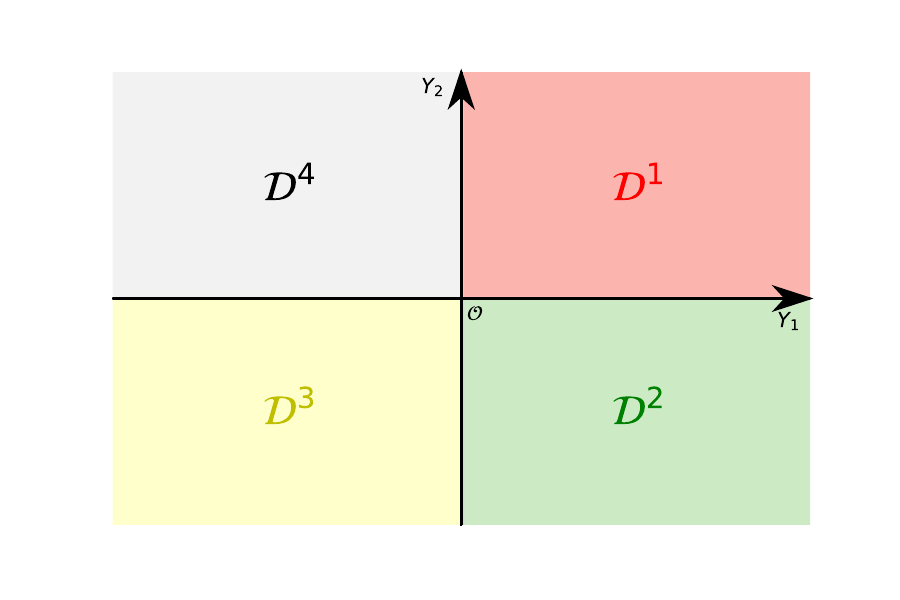}
        \caption{Domains of analyticity for the ReLU function in $\R^2$.}
        \label{RELU_domains}
    \end{subfigure}%
    ~
    \begin{subfigure}[t]{0.5\linewidth}
        \includegraphics[height=0.7\linewidth]{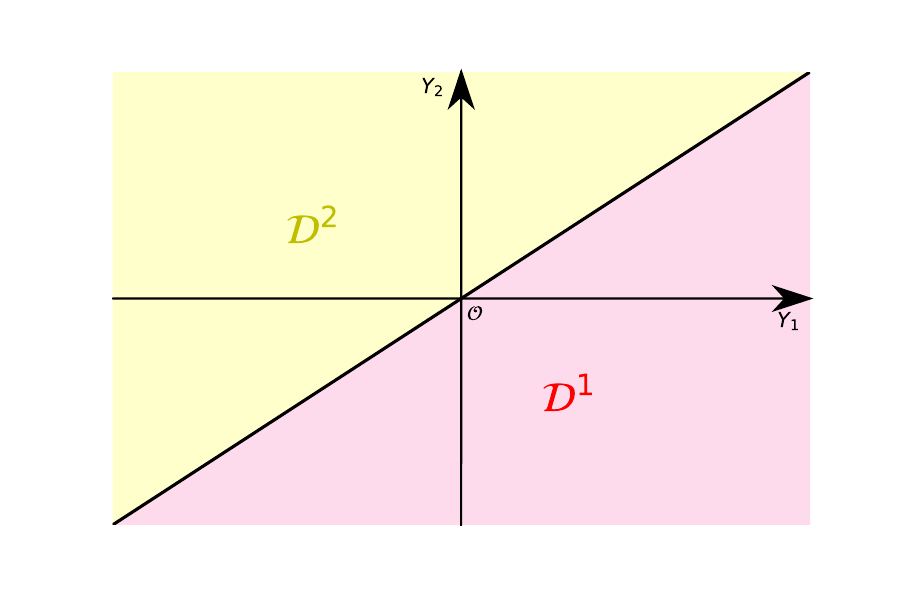}
        \caption{Domains of analyticity for max-pooling in $\R^2$.}
        \label{Max_domains}
    \end{subfigure}
    \caption{
    Illustrations of the domains of analyticity for two functions used in deep learning.
    Panel \subref{RELU_domains} represents the domains of analyticity for the ReLU function in $\R^{2}$.
    Each of the four domains $\mathcal{D}^1$, $\mathcal{D}^2$, $\mathcal{D}^3$ and $\mathcal{D}^4$ can be described by two inequalities of the form $ G_{k,1}(Y) >0 $ and $G_{k,2}(Y)>0$, where $k$ is the index of each domain.
    Here, $G_{k,1}(Y) = \pm Y_1$ and $G_{k,2}(Y) = \pm Y_2$ which are analytic functions.
    Panel \subref{Max_domains} represents the domains of analyticity for the max-pooling function in $\R^2$.
    $\mathcal{D}^1$ and $\mathcal{D}^2$ are each one described by one inequality of the form $ G_{k,1}(Y)>0 $ with $G_{1,1}(Y) = Y_1-Y_2$  and $G_{2,1}(Y) = Y_2-Y_1$, which are both analytic functions. }
    \label{fig:domains}
\end{figure*}

\subsection{Main results} \label{sec:proof:proofs}

We first state the main result, which implies \cref{thm:main-body} via \cref{cor:mmd-gan,cor:wgan,cor:gan}.
The proof depends on various intermediate results which will be established afterwards.

\begin{theorem}
 \label{thm:interversion}
Under \cref{moments,kernel_growth,Lipschitz,Diff_set}, for $\mu$-almost all $ \theta_0 \in \R^m$ the function $\theta \mapsto \E_{\PP}[ \K(h_{\theta}(X)) ]$ is differentiable at $\theta_0$, and
        \begin{align}
                    \partial_{\theta}\E_{\PP}[ \K(h_{\theta}(X)) ] = \E_{\PP}[ \partial_{\theta} \K(h_{\theta}(X))]
        ,\end{align}
where $\mu$ is the Lebesgue measure.
\end{theorem}

\begin{proof}
Let $\theta_0$ be such that the function $\theta \mapsto h_\theta(X)$ is differentiable at $\theta_0$ for $\PP$-almost all $X$.
By \cref{main_zero_measure},
this is the case for $\mu$-almost all $\theta_0$ in $\R^m$.

Consider a sequence $(\theta_n)_{n\in \n} $ that converges to $\theta_0$;
there is then an $R>0$ such that $\lVert\theta_n - \theta_0 \rVert < R$ for all $n\in\n$.
Letting $X$ be in $\R^d$, \cref{Lipschitz kernel} gives that
\begin{align}
    \lvert \K(h_{\theta_n}(X)) -  \K(h_{\theta_0}(X)) \rvert &\leq
    F(X) \lVert \theta_n - \theta_0 \rVert
\end{align}
with $\E_{\PP}[F(X)] < \infty $.
It also follows that:
\begin{align}
    \vert \partial_{\theta} K(h_{\theta_0}(X))   \vert\leq F(X)
\end{align}
for $ \PP $-almost all $X\in\R^d$.
The sequence  $\mathcal{M}_n(X) $ defined by:
\begin{align}
    \mathcal{M}_n(X)
    = \frac{1}{\lVert \theta_n - \theta_0 \rVert}
    \lvert \K(h_{\theta_n}(X)) - \K(h_{\theta_0}(X)) - \partial_{\theta} \K(h_{\theta_0}(X)) (\theta_n - \theta_0) \rvert
\end{align}
converges point-wise to $0$ and is bounded by the integrable function $2F(X)$.
Therefore by the dominated convergence theorem (\cref{thm:DC}) it follows that
\[\E_{\PP}[\mathcal{M}_n(X)] \rightarrow 0 .\]

Finally we define the sequence
\begin{align}
    \mathcal{R}_n =  \frac{1}{\Vert \theta_n - \theta_0 \Vert}
\bigl\lvert \E_{\PP}[\K(h_{\theta_n}(X))] -  \E_{\PP}[\K(h_{\theta_0}(X))] -  \E_{\PP}[\partial_{\theta}\K(h_{\theta_0}(X))](\theta_n - \theta_0) \bigr\rvert
,\end{align}
which is upper-bounded by $\E_{\PP}[\mathcal{M}_n(X)]$ and therefore converges to $0$.
By the sequential characterization of limits in \cref{seq_lemma},
it follows that $\E_{\PP}[ \K(h_{\theta}(X)) ]$ is differentiable at $\theta_0$,
and its differential is given by $\E_{\PP}[ \partial_{\theta} \K(h_{\theta}(X))]$.
\end{proof}

These corollaries of \cref{thm:interversion} apply it to specific GAN architectures.
Here we use the distribution $\Z$ to represent the noise distribution.

\begin{corollary}[WGANs] \label{cor:wgan}
    Let $\PP$ and $\Z$ be two distributions, on $\X$ and $\mathcal Z$ respectively,
    each satisfying \cref{moments} for $\alpha = 1$.
    Let $G_{\psi} : \mathcal Z \to \X$ be a generator network
    and $D_{\theta} : \X \to \R$ a critic network,
    each satisfying \cref{Lipschitz,Diff_set}.
    Then, for $\mu$-almost all $(\theta, \psi)$, we have that
    \[
        \E_{\substack{\Xset \sim \PP^m\\\Zset \sim \Z^n}}\left[ \partial_{\theta,\psi}\left[
            \frac1m \sum_{i=1}^m D_\theta(X_i)
            - \frac1n \sum_{j=1}^n D_\theta(G_{\psi}(Z_j))
        \right] \right]
        = \partial_{\theta,\psi}\left[ \E_{X \sim \PP} D_\theta(X) - \E_{Z \sim \Z} D_\theta(G_{\psi}(Z)) \right]
    .\]
\end{corollary}
\begin{proof}
    By linearity, we only need the following two results:
    \[
        \E_{X \sim \PP} \partial_\theta D_\theta(X_i)
        = \partial_\theta \E_{X \sim \PP} D_\theta(X)
        \quad\text{and}\quad
        \E_{Z \sim \Z} \partial_{\theta,\psi} D_\theta(G_\psi(Z_j))
        = \partial_{\theta,\psi} \E_{Z \sim \Z} D_\theta(G_\psi(Z))
    .\]
    The first follows immediately from \cref{thm:interversion},
    using the function $\K(U) = U$ (which clearly satisfies \cref{kernel_growth} for $\alpha = 1$).
    The latter does as well by considering that the augmented network
    $h_{(\theta,\psi)}(Z) = D_\theta(G_\psi(Z))$
    still satisifes the conditions of \cref{thm:interversion}.
\end{proof}

\begin{corollary}[Original GANs] \label{cor:gan}
    Let $\PP$ and $\Z$ be two distributions, on $\X$ and $\mathcal Z$ respectively,
    each satisfying \cref{moments} for $\alpha = 1$.
    Let $G_{\psi} : \mathcal Z \to \X$ be a generator network,
    and $D_{\theta} : \X \to \R$ a discriminator network,
    each satisfying \cref{Lipschitz,Diff_set}.
    Further assume that the output of $D$ is almost surely bounded:
    there is some $\gamma > 0$ such that
    for $\mu$-almost all $(\theta, \psi)$,
    \[
        \Pr_{X \sim \PP}\left( \gamma \le D_\theta(X) \le 1-\gamma \right) = 1
        \qquad\text{and}\qquad
        \Pr_{Z \sim \Z}\left( \gamma \le D_\theta(G_\psi(Z)) \le 1-\gamma \right) = 1
    .\]
    Then we have the following:
    \begin{align}
        \E_{X \sim \PP}\left[ \partial_{\theta}\left[ \log D_\theta(X) \right] \right]
        &= \partial_{\theta}\left[ \E_{X \sim \PP}\left[ \log D_\theta(X) \right] \right]
        \\
        \E_{Z \sim \Z}\left[ \partial_{\theta,\psi}\left[ \log\left( 1 - D_\theta(G_\psi(Z)) \right) \right] \right]
        &= \partial_{\theta,\psi}\left[ \E_{Z \sim \Z}\left[ \log\left( 1 - D_\theta(G_\psi(Z)) \right) \right] \right]
        \\
        \E_{Z \sim \Z}\left[ \partial_{\theta,\psi}\left[ \log D_\theta(G_\psi(Z)) \right] \right]
        &= \partial_{\theta,\psi}\left[ \E_{Z \sim \Z}\left[ \log D_\theta(G_\psi(Z)) \right] \right]
    .\end{align}
    Thus, by linearity,
    gradients of all the loss functions given in \citet[Section 3]{gans} are unbiased.
\end{corollary}
\begin{proof}
    The $\log$ function is real analytic and ($1/\gamma$)-Lipschitz on $(\gamma, 1-\gamma)$.
    The claim therefore follows from \cref{thm:interversion},
    using the networks $\log \circ D_\theta$, $\log \circ [x \mapsto (1 - x)] \circ D_\theta \circ G_\psi$, and $\log \circ D_\theta \circ G_\psi$ with $K(U) = U$.
\end{proof}

The following assumption about a kernel $k$ implies \cref{kernel_growth} when used as a top-level function $\K$: %
\begin{assumplist}[resume]
    \item \label{kernel_growth_pair}
    Suppose $k$ is a kernel such that there are constants $C_0$, $C_1$ where
    \begin{gather}
        \lvert k(U, V) \rvert \le C_0 \left( \left(\lVert U \rVert^2 + \lVert V \rVert^2\right)^{\alpha/2} + 1\right)
    \\  \lVert \nabla_{U,V} k(U, V) \rVert \le C_1 \left( \left( \lVert U \rVert^2 + \lVert V \rVert^2\right)^{(\alpha-1)/2} + 1\right)
    .\end{gather}
\end{assumplist}

\begin{corollary}[MMD GANs] \label{cor:mmd-gan}
    Let $\PP$ and $\Z$ be two distributions, on $\X$ and $\mathcal Z$ respectively,
    each satisfying \cref{moments} for some $\alpha \ge 1$.
    Let $k$ be a kernel satisfying \cref{kernel_growth_pair}.
    Let $G_{\psi} : \mathcal Z \to \X$ be a generator network
    and $D_{\theta} : \X \to \R$ a critic representation network
    each satisfying \cref{Lipschitz,Diff_set}.
    Then
    \[ \label{unbiased_grad_mmd}
        \E_{\substack{\Xset \sim \PP^m\\\Zset \sim \Z^n}}\left[ \partial_{\theta, \psi} \MMD_u^2(D_\theta(\Xset), D_\theta(G_{\psi}(\Zset))) \right]
     = \partial_{\theta,\psi} \MMD^2(D_\theta(\PP), D_\theta(G_{\psi}(\Z)))
    .\]
\end{corollary}
\begin{proof}
    Consider the following augmented networks:
    \begin{align}
        h^{(1)}_{(\theta, \psi)}(X,Z)   &= ( D_{\theta}(X), D_{\theta}( G_{\psi}(Z) )  )
    \\  h^{(2)}_{(\theta, \psi)}(Z,Z') &= ( D_{\theta}(G_{\psi}(Z)), D_{\theta}( G_{\psi}(Z) ) )
    \\  h^{(3)}_{\theta}(X,X') &= ( D_{\theta}(X), D_{\theta}( X ) )
    .\end{align}
    $h^{(1)}$ has inputs distributed as $\PP \times \Z$,
    which satisfies \cref{moments} with the same $\alpha$ as $\PP$ and $\Z$,
    and $h^{(1)}$ satisfies \cref{Lipschitz,Diff_set}.
    The same is true of $h^{(2)}$ and $h^{(3)}$.
    Moreover, the function
    \[
        \K\left(\begin{bmatrix}U \\ V\end{bmatrix}\right) = k(U, V)
    \]
    satisfies \cref{kernel_growth}.
    Thus \cref{thm:interversion} applies to each of $h^{(1)}$, $h^{(2)}$, and $h^{(3)}$.
    Considering the form of $\MMD_u^2$ \eqref{eq:unbiasedMMD},
    the result follows by linearity
    and the fact that $\MMD_u^2$ is unbiased \citep[Lemma 6]{mmd-jmlr}.
\end{proof}

Each of the kernels considered in this paper satisfies \cref{kernel_growth_pair} with $\alpha$ at most 2:
\begin{itemize}
\item $k^\DOT(x, y) = \langle x, y \rangle$ works with $\alpha = 2$, $C_0 = 1$, $C_1 = 1$.
\item $k^\RBF_\sigma$ of \eqref{eq:k-rbf} works with $\alpha = 2$, $C_0 = 1$,  $C_1 = \sqrt{2}\sigma^{-2}$.
\item $k^\RQ_{\alpha'}$ of \eqref{eq:k-rq} works with $\alpha = 2$, $C_0 = 1$, $C_1 = \sqrt{2}$.
\item $k^\DIST_{\rho_{\beta},0}$ of \eqref{eq:distanceKernel}, using $\rho_{\beta}(x, y) = \lVert x - y \rVert^{\beta}$ with $1 \le \beta \le 2$, works with $\alpha = \beta$, $C_0 = 3 $, $C_1 = 4\beta$.
\end{itemize}
Since the existence of a moment implies the existence of all lower-order moments by Jensen's inequality,
this finalizes the proof of \cref{thm:main-body}.

\subsection{Bounds on network growth} \label{sec:proof:growth-bounds}
The following lemmas were used in the proof of \cref{thm:interversion}.
We start by stating a result on the growth and Lipschitz properties of the network.

\begin{lemma}
Under \cref{Lipschitz}, there exist continuous functions $\theta \mapsto  b(\theta), a(\theta)$ and $ (\theta, \theta') \mapsto ( \alpha(\theta, \theta') ,  \beta(\theta, \theta'))  $ such that:
\begin{align}
    \lVert h_{\theta}(X) \rVert
    &\leq b(\theta ) + a(\theta) \lVert X \rVert\\
    \lVert h_{\theta}(X) - h_{\theta'}(X) \rVert
    &\leq \lVert\theta - \theta' \rVert
    \Big(\beta(\theta, \theta') + \alpha(\theta, \theta') \Vert X \Vert  \Big)
\end{align}
for all $X$ in $\R^{d}$ and all $ \theta, \theta' $ in $\R^{m} $.
\label{network_growth}
\end{lemma}
\begin{proof}
 Let $X$ in $\R^{d}$ and  $ \theta, \theta' $ in $\R^{m} $.
    We proceed by recursion on the nodes of the network. For $i=0$ the inequalities hold trivially  $b_0 = 0$, $ a_0=1$, $\beta_0 = 0$ and $\alpha_0 = 0$.
    Assume now that:
    \begin{align}
        \Vert h_{\theta}^{\pi(i)}(X) \Vert &\leq b_{\pi(i)}(\theta)  + a_{\pi(i)}(\theta)\Vert X \Vert\\
        \Vert h_{\theta}^{\pi(i)}(X) -   h_{\theta'}^{\pi(i)}(X)\Vert
        &\leq
         \Vert\theta - \theta' \Vert \big(\beta_{\pi(i)}(\theta, \theta') + \alpha_{\pi(i)}(\theta, \theta') \Vert X \Vert  \big)
    \end{align}
    where $a_{\pi(i)}(\theta)$, $b_{\pi(i)}(\theta)$, $\alpha_{\pi(i)}(\theta, \theta')$ and $\beta_{\pi(i)}(\theta, \theta')$ are continuous functions.
    If $i$ is a linear layer then:
    \begin{align}
        \Vert h_{\theta}^{i} \Vert
        &\leq  \Vert  g_i(W^i) \widetilde{h}^{\pi(i)}_{\theta} ) \vert \\
        &\leq \Vert  g_i\Vert \Vert W^i\Vert \Big(\Vert  h^{\pi(i)}_{\theta}\Vert +1\Big) \\
        &\leq b_{i}(\theta)  + a_{i}(\theta)\Vert X \Vert
        \end{align}
with $a_i(\theta) =  \Vert  g_i\Vert \Vert W^i\Vert a_{\pi(i)}(\theta)$ and
$b_i(\theta) =  \Vert  g_i\Vert \Vert W^i\Vert b_{\pi(i)}(\theta)$. Moreover, we have that:
\begin{align}
    \Vert h_{\theta}^i(X) - h_{\theta'}^{i}(X)  \Vert
    &= \Vert g_i(W_{i})\widetilde{h}_{\theta}^{\pi(i)}(X) -   g_i(W_i')\widetilde{h}_{\theta'}^{\pi(i)}(X)\Vert\\
    &\leq
    \Vert  g_i(W_i)\big(\widetilde{h}_{\theta'}^{\pi(i)}(X) - \widetilde{h}_{\theta}^{\pi(i)}(X) \big)\Vert   +\Vert \big( g_i(W_i-W_i')\big) \widetilde{h}_{\theta'}^{\pi(i)}(X)  \Vert\\
    &\leq
        \Vert  g_i\Vert \big( \Vert W_i\Vert \Vert h_{\theta'}^{\pi(i)}(X) - h_{\theta}^{\pi(i)}(X)\Vert  + \Vert W_i - W_i'\Vert (\Vert h_{\theta'}^{\pi(i)}(X)  \Vert + 1 )  \big)\\
    &\leq \Vert  \theta - \theta' \Vert \big(  \beta_i(\theta, \theta')  +  \alpha_i(\theta, \theta')\Vert X \Vert \big)
\end{align}
with:
\begin{align}
    \alpha_i(\theta, \theta') &=   \Vert g_i\Vert \big( ( \Vert W_i \Vert + \Vert  W_i'\Vert  ) \alpha_{\pi(i)}(\theta, \theta')  + ( a_{\pi(i) }(\theta) +  a_{\pi(i) }(\theta')) \big)\\
    \beta_i(\theta, \theta') &=   \Vert g_i\Vert \big( ( \Vert W_i \Vert + \Vert  W_i'\Vert  ) \beta_{\pi(i)}(\theta, \theta')  + ( b_{\pi(i) }(\theta) +  b_{\pi(i) }(\theta') )+ 1 \big)
.\end{align}

When $i$ is not a linear layer, then by \cref{Lipschitz} $f_i$ is $M$-Lipschitz. Thus we can directly get the needed functions by recursion:  $\alpha_i = M\alpha_{\pi(i)}  $, $\beta_i = M\beta_{\pi(i)}  $, $a_i = Ma_{\pi(i)}  $ and $b_i = Mb_{\pi(i)}$.
\end{proof}

\begin{lemma}
Let $R$ be a positive constant and $\theta\in \R^m$. Under \cref{kernel_growth,moments,Lipschitz}, the following hold for all $\theta'\in B(\theta,R)$ and all $ X $ in $\R^d$ :
    \begin{align}
        \vert \K(h_{\theta'}(X)) - \K(h_{\theta}(X)) \vert \leq F(X) \Vert \theta- \theta' \Vert
    \end{align}
    with
    $\E_{\PP}[F(X)]< \infty$.
    \label{Lipschitz kernel}
\end{lemma}

\begin{proof}
We will first prove the following inequality:
\begin{align}
    \vert \K(U) - \K(V) \vert \le C_1
         \bigl( \left(\Vert U-V \Vert + \Vert V \Vert \right)^{\alpha-1} +1 \bigr)
        \Vert U-V\Vert
\end{align}
for all $ U $ and $V$ in $\R^{L}$.

Let $t$ be in $[0,1]$ and define the function $f$ by
\begin{align}
    f(t) = \K(tU + (1-t)V  )
\end{align}
Then $f(0) = \K(V)$ and $ f(1) =  \K(U) $. Moreover, $f$ is differentiable and its derivative is given by:
\begin{align}
    f'(t) = \nabla \K( (t U+ (1-t)V) )(U-V)
\end{align}
Using \cref{kernel_growth} one has that:
\begin{align}
    \vert f'(t)\vert
    &  = \bigl\lVert \nabla \K( (t U+ (1-t)V )( U-V) \bigr\rVert
\\  &\le \Vert \nabla \K( t U+ (1-t)V )\Vert \Vert (U-V) \Vert
\\  &\le C_1
         \bigl( \Vert t(U-V) + V  \Vert^{\alpha-1} + 1\bigr)
         \Vert U-V\Vert \\  &\le C_1
         \bigl( \left(\Vert U-V \Vert + \Vert V \Vert \right)^{\alpha-1}  + 1  \bigr)
        \Vert U-V\Vert
.\end{align}
The conclusion follows using the mean value theorem.
Now choosing $U = h_{\theta'}(X)$ and $V = h_{\theta}(X)$ one gets the following:
\begin{align}
    \vert \K(h_{\theta'}(X)) - \K(h_{\theta}(X)) \vert \le
    C_1( \left(\Vert h_{\theta'}(X)-h_{\theta}(X) \Vert + \Vert h_{\theta}(X) \Vert \right)^{\alpha-1}   + 1\bigr)
        \Vert h_{\theta'}(X)-h_{\theta}(X)\Vert
\end{align}
Under \cref{Lipschitz}, it follows by \cref{network_growth} that:
\begin{align}
    \Vert h_{\theta'}(X)-h_{\theta}(X) \Vert &\leq
    \Big(\beta(\theta, \theta') + \alpha(\theta, \theta') \Vert X \Vert  \Big) \Vert \theta  - \theta' \Vert \\
    \lVert h_{\theta}(X) \rVert
    &\leq b(\theta ) + a(\theta) \lVert X \rVert\\
\end{align}

The functions $a$, $b$, $\alpha$, $\beta$ defined in \cref{network_growth} are continuous,
and hence all bounded on the ball $B(\theta, R$);
choose $D > 0$ to be a bound on all of these functions.
It follows after some algebra that
\begin{align}
    \left\lvert \K(h_{\theta'}(X)) - \K(h_{\theta}(X)) \right\rvert \le
    C_1 (D^{\alpha} (R+1)^{\alpha-1} (1  + \Vert X\Vert )^{\alpha}  + D(1+\Vert X \Vert )) \lVert \theta' - \theta \rVert
.\end{align}
Set $F(X) =  C_1 (D^{\alpha} (R+1)^{\alpha-1} (1  + \Vert X\Vert )^{\alpha}  + D(1+\Vert X \Vert ))$.
Since $\alpha\geq1$, $t \mapsto (1 + t^{1/\alpha})^\alpha$ is concave on $t \ge 0$, and so we have that
     \[
     \E\left[ \left( 1 + \lVert X \rVert \right)^{\alpha} \right]
     \le \left(1 + \E\left[\lVert X \rVert^{\alpha}\right]^{1/\alpha} \right)^\alpha
     < \infty
     \]
    via Jensen's inequality and \cref{moments}. We also have $\E\left[ 1 + \lVert X \rVert \right] < \infty $ by the same assumption. Thus $F(X)$ is integrable.
\end{proof}

\begin{lemma}
Let $ f:\R^m \rightarrow \R $ be a real valued function and $g$ a vector in $\R^m$ such that:
\[
    \frac{1}{\Vert \theta_n - \theta_0 \Vert} \lvert f(\theta_n) - f(\theta_0) - g \cdot (\theta_n - \theta_0)  \rvert
    \rightarrow 0
\]
for all sequences $(\theta_n)_{n\in \n}$ converging towards $\theta_0$ with $\theta_n \neq \theta_0$.
Then $f$ is differentiable at $\theta_0$, and its differential is $g$.
\label{seq_lemma}
\end{lemma}
\begin{proof}
	Recall the definition of a differential: $g$ is the differential of $f$ at $\theta_0$ if 
	\[
\lim_{h\rightarrow 0} \frac{1}{\Vert h \Vert} \left\lvert f(\theta_0 + h) - f(\theta_0) - g \cdot h \right\rvert  = 0
    .\]
The result directly follows from the sequential characterization of limits.
\end{proof}

\subsection{Critical parameters have zero measure}

The last result required for the proof of \cref{thm:interversion}
is \cref{main_zero_measure}.
We will first need some additional notation.

For a given node $i$,
we will use the following sets of indices to denote ``paths'' through the network's computational graph:
\begin{align}\label{indices_set_P}
P := \big\{ (i,k,s) \in \n^3 \mid i \in \{0\} \cup [L], k \in [K_i], s \in [S_{i,k}] \big\}
  \end{align}
\begin{align}\label{eq:ancestor}
    \neg i &= \{ (j,k,s) \in P \mid j \text{ is an ancestor of } i \text{ or } j=i \}\\
    \neg \pi(i) &= \bigcup_{j \in \pi(i)} \neg j = \{ (j,k,s) \in P \mid j \text{ is an ancestor of } i \}\\
    \partial i &= \{ (i,u,s) \in P \}
.\end{align}
Note that $\partial i \subseteq \neg i$,
and that $\neg i = \partial i \cup \neg \pi(i)$.

If $a(i)$ is the set of ancestors of node $i$, we define a backward trajectory starting from node $i$ as an element $q$ of the form:
\begin{align}\label{trajectory}
    q := (j,k_j)_{ j \in a(i)\cup \{ i \} }
\end{align}
where $k_j$ are integers in $[K_j]$. We call $T(i)$ the set of such trajectories for node $i$.

For $p\in P$ of the form $p = (i,k,s)$,
the set of parameters for which we lie on the boundary of $p$ is
\begin{align}\label{eq:S_P}
    S^p = \{ \theta\in \R^m \mid G_{i,k,s}(h_{\theta}^{\pi(i)}) = 0 \}
.\end{align}
We also denote by $\partial S^p$ the boundary of the set $S^p$.
If $Q$ is a subset of $P$, we use the following notation for convenience:
\[ \label{eq:S_Q_boundary}
    S^{Q} := \bigcup_{q\in Q} S^{q}
    ,\qquad
    \partial S^{Q} := \bigcup_{q\in Q} \partial S^{q}
.\]

For a given $\theta_0 \in \R^m$, the set of input vectors $X\in \R^d$ such that $h_{\theta_0}$ is not differentiable is
\begin{align}
    \mathcal{N}(\theta_0) = \Big\{  X\in \R^d \mid \theta \mapsto  h_{\theta}(X) \text{ is not differentiable at } \theta_0  \Big\}
    \label{eq:def-n-theta}
.\end{align}

Consider a random variable $X$ in the input space $\R^d$, following the distribution $\PP$.
For a given distribution $\PP$, we introduce the following set of "critical" parameters:
\begin{align}
    \Theta_{\PP} = \Big\{ \theta \mid \PP( \mathcal{N}(\theta) ) > 0  \Big\}
    \label{eq:def-theta-p}
.\end{align}
This is the set of parameters $\theta$ where the network is not differentiable
for a non-negligible set of datasets $X$.

Finally, for a given $X\in \R^d$, set of parameters for which the network is not differentiable is
\begin{align}
    \Theta_X &= \Big\{  \theta_0\in \R^m \mid \theta \mapsto h_{\theta}(X) \text{ is non-differentiable in } \theta_0 \Big\}
    \label{eq:def-theta-x}
.\end{align}

We are now ready to state and prove the remaining result.

\begin{prop} \label{main_zero_measure}
Under \cref{Diff_set}, the set $ \Theta_{\PP} $ has $0$ Lebesgue measure for any distribution $\PP$.
\end{prop}
\begin{proof}
    Consider the following two sets:
    \begin{align}
        D &= \Big\{ (\theta, X)\in \R^m \times\R^d \mid \theta \in \Theta_{\PP} \textit{ and } X\in \mathcal{N}(\theta)  \Big\}\\
        Q &= \Big\{  (\theta, X)\in \R^m \times\R^d \mid \theta \in \Theta_{X} \Big\}
    .\end{align}
    By virtue of Theorem I in \cite{Zahorski:1946,Piranian:1966}, it follows that the set of non-differentiability of continuous functions is measurable. It is easy to see then, that $D$ and $Q$ are also measurable sets since the network is continuous.
    Note that we have the inclusion $ D \subseteq Q $. We endow the two sets with the product measure $\nu := \mu \times \PP$, where $\mu$ is the Lebesgue measure.
    Therefore $ \nu(D) \leq \nu(Q)$.
    On one hand, Fubini's theorem tells us:
    \begin{align}
        \nu(Q) &= \int_{\R^d} \int_{\Theta_X} \ud\mu(\theta) \, \ud\PP(X)\\
                & = \int_{\R^d} \mu(\Theta_X) \ud\PP(X)
    .\end{align}
    By \cref{zero_measure}, we have that $ \mu(\Theta_X) = 0$;
    therefore $\nu(Q) = 0$ and hence $ \nu(D)  = 0$.
    On the other hand, we use again Fubini's theorem for $ \nu(D) $ to write:
    \begin{align}
        \nu(D) &= \int_{\Theta_{\PP}}\int_{\mathcal{N}(\theta)} \ud\PP(X) \, \ud\mu(\theta)\\
                &= \int_{\Theta_{\PP}} \PP(\mathcal{N}(\theta)) \, \ud\mu(\theta)
    .\end{align}
    For all $\theta \in \Theta_\PP$, we have $\PP(\N(\theta)) > 0$ by definition.
    Thus $\nu(D) = 0$ implies that $\mu(\Theta_\PP) = 0$.
\end{proof}

\begin{lemma} \label{zero_measure}
Under \cref{Diff_set}, for any $X$ in $\R^{d}$, the set $\Theta_X $ has $0$ Lebesgue measure: $\mu(\Theta_X) = 0$.
\end{lemma}
\begin{proof}
We first show that $\Theta_X \subseteq \partial S^P$,
which was defined by \eqref{eq:S_Q_boundary}.

Let $\theta_0$ be in $\Theta_X$.
By \cref{Diff_set}, it follows that $ \theta_0\in S^P$.
Assume for the sake of contradiction that $\theta_0 \notin\partial S^P$.
Then applying \cref{smooth_interior} to the output layer, $i = L$,
implies that there is a real analytic function $f(\theta)$
which agrees with $h_\theta$ on all $\theta \in B(\theta_0, \eta)$
for some $\eta > 0$.
Therefore the network is differentiable at $\theta_0$,
contradicting the fact that $\theta_0 \in \Theta_X$.
Thus $\Theta_X \subseteq \partial S^P$.

\Cref{locally_null_sets} then establishes that $\mu(\partial S^P) = 0$,
and hence $\mu(\Theta_X) = 0$.
\end{proof}

\begin{lemma}
    Let $i$ be a node in the graph.
    Under \cref{Diff_set},
    if $\theta \in \R^m\setminus \partial S^{\neg i}$,
    then there exist $\eta>0$ and a trajectory $q\in T(i) $
    such that $h_{\theta'}^{i} = f^{q}(\theta')$
    for all $\theta'$ in the ball $B(\theta, \eta)$.
    Here $f^{q}$ is the real analytic function on $\R^m$
    defined with the same structure as $h_\theta$,
    but replacing each nonlinear $f_j$
    with the analytic function $f_j^{k_j}$ for $(j, k_j) \in q$.
    \label{smooth_interior}
\end{lemma}
\begin{proof}
We proceed by recursion on the nodes of the network.
If $i=0$, we trivially have $  h_{\theta}^{0} = X$, which is real analytic on $\R^m $.
Assume the result for  $\neg \pi(i)$ and let $\theta \in \R^{m}\setminus \partial S^{\neg i}  $.
In particular $ \theta \in \R^m \setminus \partial S^{\neg \pi(p) } $.
By the recursion assumption, we get:
\begin{align}   \label{eq:recursion_assumption_smooth_interior}
    (\exists \eta >0) ( \exists q \in T(i))  (\forall \theta' \in B(\theta, \eta))
    \qquad h_{\theta'}^{\pi(i) } = f^{q}(\theta')
\end{align}
with $f^{q}$ real analytic in $\R^m$.

If $\theta \notin S^{\partial i}$, then there is some sufficiently small $\eta' > 0$ such that $B(\theta, \eta')$ does not intersect $S^{\partial i}$.
Therefore, by \cref{Diff_set}, there is some $k \in [K_i]$ such that
$h_{\theta'}^{i} = f_{i}^{k}( h_{\theta'}^{\pi(i)} )$
for all $ \theta' \in B(\theta, \eta')$,
where $f_{i}^{k}$ is one of the real analytic functions defining $f_i$.
By \eqref{eq:recursion_assumption_smooth_interior} we then have
\[ \label{eq:recursion_implication_smooth_interior}
    h_{\theta'}^{i}  = f_{i}^{k}(f^{q}(\theta'))
    \quad\forall \theta' \in B(\theta, \min(\eta, \eta'))
.\]

Otherwise, $\theta \in S^{\partial i}$.
Then, noting that by assumption $\theta \notin \partial S^{\partial i}$,
it follows that for small enough $\eta' > 0$, we have $B(\theta, \eta') \subseteq S^{\partial i}$.
Denote by $A$ the set of index triples $p \in \partial i$ such that $ \theta \in S^p$;
$A$ is nonempty since $\theta \in S^{\partial i}$.
Therefore $\theta \in \bigcap_{p\in A} S^p$,
and $\theta \notin \bigcup_{p\in A^c} S^p$.
We will show that for $\eta'$ small enough,
$B(\theta, \eta') \subseteq \bigcap_{p\in A} S^p$.
Assume for the sake of contradiction
that there exists a sequence of (parameter, index-triple) pairs $(\theta_n, p_n)$
such that $p_n \in A^c$, $\theta_n \in S^{p_n}$, and $\theta_n \to \theta$.
$p_n$ is drawn from a finite set and thus has a constant subsequence,
so we can assume without loss of generality that $p_n = p_0$ for some $p_0 \in A^c$.
Since $S^{p_0}$ is a closed set by continuity of the network and $G_{p_0}$,
it follows that $ \theta \in S^{p_0} $ by taking the limit.
This contradicts the fact that  $\theta \notin  \bigcup_{p\in A^c} S^{p}$.
Hence, for $\eta'$ small enough, $B(\theta, \eta') \subseteq \bigcap_{p\in A} S^p$.
Again, by \cref{Diff_set} there is a $k \in [K_i]$ satisfying \eqref{eq:recursion_implication_smooth_interior}.

By setting $f^{q_0} = f_{i}^k( f^{q})$ with  $ q_0 = ( (i,k)\oplus q )$, where $\oplus$ denotes concatenation,
it finally follows that $ h_{\theta'}^{i} = f^{q_0}(\theta') $ for all $ \theta' $ in $B(\theta, \min(\eta, \eta'))$,
and $f^{q_0}$ is the real analytic function on $\R^{m}$ as described.
\end{proof}

\begin{lemma} \label{locally_null_sets}
    Under \cref{Diff_set}, \[ \mu( \partial S^P ) = 0 .\]
\end{lemma}
\begin{proof}
    We will proceed by recursion.
    For $i=0 $ we trivially have $ \partial S^{\neg 0}  = \emptyset$, thus $\mu(\partial S^{\neg 0}) = 0$.
    Thus assume that \[ \mu(\partial S^{\neg \pi(i)}) = 0. \]

    For $s = (p,q)$,
    the pair of an index triple $p\in \partial i$ and a trajectory $q\in T(i)$,
    define the set
    \[ \mathcal{M}_s = \{ \theta \in \R^m \mid G_{p}(f^q(\theta)) = 0 \} ,\]
    where $f^q$ is the real analytic function defined in \cref{smooth_interior}
    which locally agrees with $h^{\pi(i)}_\theta$.

    We will now prove that for any $\theta$ in
    $\partial S^{\partial i}\setminus \partial S^{\neg \pi(i)}$,
    there exists $s\in \partial i \times T(i)$ such that
    $\theta \in\mathcal{M}_{s}$ and $\mu(\mathcal{M}_s) = 0$.
    We proceed by contradiction.

    Let $\theta \in \partial S^{\partial i }\setminus \partial S^{\neg \pi(i)}$;
    then for small enough $\eta>0$,
    $B(\theta, \eta) \subseteq \R^m \setminus \partial S^{\neg \pi(i)}$.
    By \cref{smooth_interior}, there is a trajectory $q \in T(i)$ such that
    \[ \label{eq:locally_equal}
        h_{\theta'}^{\pi(i)} = f^{q}(\theta') \qquad \forall  \theta'\in B(\theta, \eta)
    .\]
    Moreover, since $\theta \in \partial S^{\partial i}$,
    there exists $p \in \partial i$ such that $G_p(h^{\pi(i)}_{\theta}) = 0$.
    This means that for $s = (p, q)$, we have $ \theta \in \mathcal{M}_{s} $.
    If $\mu(\mathcal{M}_{s}) > 0$, then by \cref{manifold_dim} $\mathcal{M}_{s} = \R^m$,
    hence we would have $B(\theta, \eta) \subseteq \mathcal{M}_{s}$.
    By \eqref{eq:locally_equal} it would then follow that $B(\theta, \eta) \subseteq S^{\partial i}$.
    This contradicts the fact that $\theta$ is in $\partial S^{\partial i}$,
    and hence $\mu(\mathcal M_s) = 0$.

    We have shown that $\partial S^{\partial i }\setminus \partial S^{\neg \pi(i)} \subseteq \bigcup_{s\in A}  \mathcal{M}_{s}$,
    where the sets $ \mathcal{M}_{s}$ have zero Lebesgue measure
    and $A \subseteq P \times \bigcup_{j=0}^L T(j)$ is finite.
    This implies:
    \[\mu( \partial S^{\partial i } \setminus \partial S^{\neg \pi(i)}) \leq \sum_{s\in A} \mu( \mathcal{M}_s) = 0 .\]
    Using the recursion assumption $\mu(\partial S^{\neg \pi(i)}) = 0 $, one concludes that $\mu(\partial S^{\neg i}) = 0 $. Hence for the last node $L$, recalling that $ \neg L = P $ one gets      $\mu(\partial S^{P}) = 0 $.
\end{proof}

\begin{lemma}
    Let $\theta \mapsto F(\theta): \R^m \rightarrow \R $ be a real analytic function on $\R^m$ and define the set:
    \[
    \mathcal{M} := \{ \theta \in \R^m \mid  F(\theta) = 0 \}
    .\]
    Then either $ \mu(\mathcal{M})  = 0 $ or $F$ is identically zero.
\label{manifold_dim}
\end{lemma}
\begin{proof}
    This result is shown e.g.\ as Proposition 0 of \citet{Mityagin:2015}.
\end{proof}

\section{FID estimator bias} \label{appendix:fid-bias}

We now further study the bias behavior of the FID estimator \citep{fid} mentioned in \cref{sec:evaluation}.

We will refer to the Fr\'echet Inception Distance between two distributions,
letting $\mu_\bP$ denote the mean of a distribution $\bP$
and $\Sigma_\bP$ its covariance matrix,
as
\[
  \FID(\bP, \bQ)
  = \left\lVert \mu_\bP - \mu_\bQ \right\rVert^2
  + \tr(\Sigma_\bP) + \tr(\Sigma_\bQ) - 2 \tr\left( \left(\Sigma_\bP \Sigma_\bQ\right)^{\frac12} \right)
.\]
This is motivated because it coincides with the Fr\'echet (Wasserstein-2) distance between normal distributions.
Although the Inception coding layers to which the FID is applied are not normally distributed,
the FID remains a well-defined pseudometric between arbitrary distributions whose first two moments exist.

The usual estimator of the FID
based on samples $\{X_i\}_{i=1}^m \sim \bP^m$ and $\{Y_j\}_{j=1}^n \sim \bP^n$
is the plug-in estimator.
First, estimate the mean and covariance with the standard estimators:
\[
  \hat\mu_X = \frac1n \sum_{i=1}^n X_i
  ,\qquad
  \hat\Sigma_X = \frac{1}{n-1} \sum_{i=1}^n (X_i - \hat\mu_X) (X_i - \hat\mu_X)\tp
.\]
Letting $\hat\bP_X$ be a distribution matching these moments,
e.g. $\N\left( \hat\mu_X, \hat\Sigma_X \right)$,
the estimator is given by
\[
  \FID\left( \hat\bP_X, \hat\bQ_Y \right)
  = \left\lVert \hat\mu_X - \hat\mu_Y \right\rVert^2
  + \tr(\hat\Sigma_X) + \tr(\hat\Sigma_Y) - 2 \tr\left( \left(\hat\Sigma_X \hat\Sigma_Y\right)^{\frac12} \right)
.\]

In \cref{appendix:fid-bias:1d-normal,appendix:fid-bias:relu},
we exhibit two examples where
$\FID(\bP_1, \bQ) < \FID(\bP_2, \bQ)$,
but the estimator $\FID(\hat\bP_1, \bQ)$ is usually greater than $\FID(\hat\bP_2, \bQ)$
with an equal number of samples $m$ from $\bP_1$ and $\bP_2$,
for a reasonable number of samples.
(As $m \to \infty$, of course, the estimator is consistent, and so the order will eventually be correct.)
We assume here an infinite number of samples $n$ from $\bQ$ for simplicity;
this reversal of ordering is even easier to obtain when $n = m$.
It is also trivial to achieve when the number of samples from $\bP_1$ and $\bP_2$ differ,
as demonstrated by \cref{fig:fid-bias}.

Note that \cref{appendix:fid-bias:1d-normal,appendix:fid-bias:relu}
only apply to this plug-in estimator of the FID;
it remains conceivable that there would be some other estimator for the FID which is unbiased.
\Cref{appendix:fid-bias:no-unbiased} shows that this is not the case:
there is no unbiased estimator of the FID.

\subsection{Analytic example with one-dimensional normals} \label{appendix:fid-bias:1d-normal}
We will first show that the estimator can behave poorly even with very simple distributions.

When $\bP = \N(\mu_\bP, \Sigma_\bP)$
and $\bQ = \N(\mu_\bQ, \Sigma_\bQ)$,
it is well-known that
\[
    \hat\mu_X \sim \N\left(\mu_\bP, \frac1m \Sigma_\bP \right)
    \quad\text{and}\quad
    (m-1) \hat\Sigma_X \sim \W\left( \Sigma_\bP, m-1 \right)
,\]
where $\W$ is the Wishart distribution.
Then we have
\begin{align}
     \E_{XY}\left[ \lVert \hat\mu_X - \hat\mu_Y \rVert^2 \right]
  &= \lVert \mu_\bP - \mu_\bQ \rVert^2 + \frac1m \tr(\Sigma_\bP) + \frac1n \tr(\Sigma_\bQ)
\\   \E_{XY}\left[ \tr(\hat\Sigma_X) + \tr(\hat\Sigma_Y) \right]
  &= \tr(\Sigma_\bP) + \tr(\Sigma_\bQ)
.\end{align}
The remaining term $\E \tr\left(\left(\hat\Sigma_X \hat\Sigma_Y\right)^{\frac12}\right)$ is more difficult to evaluate,
because we must consider the correlations across dimensions of the two estimators.
But if the distributions in question are one-dimensional,
denoting $\Sigma_\bP = \sigma_\bP^2$ and $\hat\Sigma_X = \hat\sigma_X^2$,
the matrix square root becomes simple:
\[
     \E_{XY}\left[ \tr\left( \left(\hat\Sigma_X \hat\Sigma_Y \right)^\frac12 \right) \right]
   = \E_{X}\left[ \hat\sigma_X \right] \E_Y\left[ \hat\sigma_Y \right]
.\]
Since $\frac{\sqrt{m-1}}{\sigma_\bP} \hat\sigma_X \sim \chi_{m-1}$,
we get that
\[
    \E_X[ \hat\sigma_X ]
  = \frac{\sigma_\bP}{\sqrt{m-1}} \sqrt{2} \frac{\Gamma\left(\frac{m}{2}\right)}{\Gamma\left(\frac{m-1}{2}\right)}
  = d_m \sigma_\bP
  \quad\text{ where }
  d_m := \frac{\sqrt{2} \operatorname{\Gamma}\left(\frac{m}{2}\right)}{\sqrt{m-1} \operatorname{\Gamma}\left(\frac{m-1}{2}\right)}
.\]
Thus the expected estimator for one-dimensional normals becomes
\[
    \E_{X,Y}\left[ \FID\left( \hat\bP_X, \hat\bQ_Y \right) \right]
    = (\mu_\bP - \mu_\bQ)^2 + \frac{m+1}{m} \sigma_\bP^2 + \frac{n+1}{n} \sigma_\bQ^2
    - 2 d_m d_n \sigma_\bP \sigma_\bQ
\label{eq:exp-fid-1d}
.\]

Now, consider the particular case
\[
  \bP_1 := \N\left( 0, \left(1 - \frac1m\right)^2 \right)
  \qquad
  \bP_2 := \N(0, 1)
  \qquad
  \bQ := \bP_2 = \N(0, 1)
.\]
Clearly
\[
  \FID(\bP_1, \bQ) = \frac{1}{m^2}
  >
  0 = \FID(\bP_2, \bQ)
.\]
But letting $n \to \infty$ in \eqref{eq:exp-fid-1d} gives
\begin{align}
     \E_{X \sim \bP_1^m} &\FID\left(\hat\bP_X, \bQ\right)
   - \E_{Y \sim \bP_2^m} \FID\left(\hat\bP_Y, \bQ\right)
   = \frac1m\left( \frac{1}{m^2} - \frac{1}{m} + 2 (d_m - 1) \right)
   < 0
,\end{align}
where the inequality follows because
$\frac{1}{m^2} < \frac1m$ and $d_m < 1$ for all $m \ge 2$.
Thus we have the undesirable situation
\[
  \FID(\bP_1, \bQ) > \FID(\bP_2, \bQ)
  \quad\text{but}\quad
  \E_{X \sim \bP_1^m} \FID\left(\hat\bP_X, \bQ\right)
  < \E_{Y \sim \bP_2^m} \FID\left(\hat\bP_Y, \bQ\right)
.\]

\subsection{Empirical example with high-dimensional censored normals} \label{appendix:fid-bias:relu}

The example of \cref{appendix:fid-bias:1d-normal},
though indicative in that the estimator can behave poorly even with very simple distributions,
is somewhat removed from the situations in which we actually apply the FID.
Thus we now empirically consider a more realistic setup.

First, as noted previously,
the hidden codes of an Inception coding network are not well-modeled by a normal distribution.
They are, however, reasonably good fits to a censored normal distribution $\relu(X)$,
where $X \sim \N(\mu, \Sigma)$
and $\relu(X)_i = \max(0, X_i)$.
Using results of \citet{rosenbaum},
it is straightforward to derive the mean and variance of $\relu(X)$ \citep{censored-normal},
and hence to find the population value of $\FID(\relu(X), \relu(Y))$.

Let $d = 2048$, matching the Inception coding layer, and consider
\[
  \bP_1 = \relu(\N(\mathbf 0, I_d))
  \qquad
  \bP_2 = \relu(\N(\mathbf 1, .8 \Sigma + .2 I_d))
  \qquad
  \bQ = \relu(\N(\mathbf 1, I_d))
\]
where $\Sigma = \frac{4}{d} C C^T$,
with $C$ a $d \times d$ matrix whose entries are chosen iid standard normal.
For one particular random draw of $C$,
we found that
$\FID(\bP_1, \bQ) \approx 1123.0 > 1114.8 \approx \FID(\bP_2, \bQ)$.
Yet with $m = 50\,000$ samples,
$\FID(\hat\bP_1, \bQ) \approx 1133.7 \text{ (sd $0.2$)} < 1136.2 \text{ (sd $0.5$)} \approx \FID(\hat\bP_2, \bQ)$.
The variance in each estimate was small enough that of 100 evaluations,
the largest $\FID(\hat\bP_1, \bQ)$ estimate was less than the smallest $\FID(\hat\bP_2, \bQ)$ estimate.
At $m = 100\,000$ samples, however, the ordering of the estimates was correct in each of 100 trials,
with $\FID(\hat\bP_1, \bQ) \approx 1128.0$ (sd $0.1$)
and $\FID(\hat\bP_2, \bQ) \approx 1126.4$ (sd $0.4$).
This behavior was similar for other random draws of $C$.

This example thus gives a case where,
for the dimension and sample sizes at which we actually apply the FID
and for somewhat-realistic distributions,
comparing two models based on their FID estimates will not only not reliably give the right ordering
-- with relatively close true values and high dimensions, this is not too surprising
-- but, more distressingly, will \emph{reliably give the wrong answer},
with misleadingly small variance.
This emphasizes that unbiased estimators, like the natural KID estimator,
are important for model comparison.

\subsection{Non-existence of an unbiased estimator} \label{appendix:fid-bias:no-unbiased}
We can also show, using the reasoning of \citet{unbiased-convex} that we also employed in \cref{thm:ipm-always-biased},
that there is no estimator of the FID which is unbiased for all distributions.

Fix a target distribution $\QQ$,
and define the quantity $F(\PP) = \FID(\PP, \QQ)$.
Also fix two distributions $\PP_0 \ne \PP_1$.
Suppose there exists some estimator
$\hat F(\Xset)$
based on a sample of size $n$
for which
\[
  \E_{\Xset \sim \PP^n}\left[ \hat F(\Xset) \right] = F(\PP)
\]
for all $\PP \in \left\{ (1 - \alpha) \PP_0 + \alpha \PP_1 \mid \alpha \in [0, 1] \right\}$.

Now consider the function
\begin{align}
     R(\alpha)
  &= F(\alpha \PP_1 + (1 - \alpha) \PP_2)
\\&= \int_{x_1} \!\cdots\! \int_{x_n} \hat F(\Xset)
     \,\ud\left[ \alpha \PP_1 + (1 - \alpha) \PP_1 \right](x_1)
     \cdots
     \ud\left[ \alpha \PP_1 + (1 - \alpha) \PP_1 \right](x_n)
\\&= \int_{x_1} \!\cdots\! \int_{x_n} \hat F(\Xset)
     \left[ \alpha \,\ud\PP_1(x_1) + (1 - \alpha) \,\ud\PP_2(x_1) \right]
     \cdots
     \left[ \alpha \,\ud\PP_1(x_n) + (1 - \alpha) \,\ud\PP_2(x_n) \right]
\\&= \alpha^n \E_{\Xset \sim \PP_1^n}\left[ \hat F(\Xset) \right]
   + \dots
   + (1 - \alpha)^n \E_{\Xset \sim \PP_2^n}\left[ \hat F(\Xset) \right]
.\end{align}
This function $R(\alpha)$ is therefore a polynomial in $\alpha$ of degree at most $n$.

But let's consider the following one-dimensional case:
\[
  \PP_0 = \N(\mu_0, \sigma_0^2)
  \qquad
  \PP_1 = \N(\mu_1, \sigma_1^2)
  \qquad
  \QQ   = \N(\mu  , \sigma  ^2)
.\]
The mean and variance of $(1 - \alpha) \PP_0 + \alpha \PP_1$
can be written as
\begin{align}
  \mu_\alpha
  &= (\mu_1 - \mu_0) \alpha + \mu_0
  \\
  \sigma^2_\alpha
  &= - (\mu_0 - \mu_1)^2 \alpha^2
   + \left( (\mu_0 - \mu_1)^2 - \sigma_0^2 + \sigma_1^2 \right) \alpha
   + \sigma_0^2
.\end{align}
Thus
\[
  R(\alpha)
  = \left( \mu_\alpha - \mu \right)^2
  + \sigma^2_\alpha + \sigma^2
  - 2 \sigma \sigma_\alpha
.\]
Note that $\left( \mu_\alpha - \mu \right)^2 + \sigma^2_\alpha + \sigma^2$
is a quadratic function of $\alpha$.
However,
$\sigma_\alpha$ is polynomial in $\alpha$
only in the trivial case when $\PP_0 = \PP_1$.
Thus $R(\alpha)$ is not a polynomial when $\PP_0 \ne \PP_1$,
and so
no estimator of the FID to an arbitrary fixed normal distribution $\QQ$
can be unbiased
on any class of distributions which includes two-component Gaussian mixtures.

There is also no unbiased estimator is available in the two-sample setting, where $\QQ$ is also unknown,
by the same trivial extension to this argument as in \cref{thm:ipm-always-biased}.

Unfortunately, this type of analysis can tell us nothing about whether there exists an estimator which is unbiased on normal distributions.
Given that the distributions used for the FID in practice are clearly not normal, however,
a practical unbiased estimator of the FID is impossible.

\section{Comparison of evaluation metrics' resilience to noise} \label{appendix:noise}
We replicate here the experiments of \citeauthor{fid}'s Appendix 1,
which examines the behavior of the Inception and FID scores as images are increasingly ``disturbed,''
and additionally consider the KID.
As the ``disturbance level'' $\alpha$ is increased,
images are altered more from the reference distribution.
\Cref{fig:disturb:gaussian-noise,fig:disturb:gaussian-blur,fig:disturb:black-rectangles,fig:disturb:swirl,fig:disturb:salt-pepper,fig:disturb:imagenet} show the FID, KID, and negative (for comparability) Inception score for both CelebA (left) and CIFAR-10 (right);
each score is scaled to $[0, 1]$ to be plotted on one axis,
with minimal and maximal values shown in the legend.

Note that \citeauthor{fid}\ compared means and variances computed on
$50\,000$ random disturbed CelebA images
to those computed on the full $200\,000$ dataset;
we instead use the standard train-test split,
computing the disturbances on the $160\,000$-element training set
and comparing to the $20\,000$-element test set.
In this (very slightly) different setting,
we find the Inception score to be monotonic with increasing noise on more of the disturbance types than did \citet{fid}.
We also found similar behavior on the CIFAR-10 dataset,
again comparing the noised training set (size $50\,000$)
to the test set (size $10\,000$).
This perhaps means that the claimed non-monotonicity of the Inception score is quite sensitive to the exact experimental setting;
further investigation into this phenomenon would be intriguing for future work.

\begin{figure}
  \centering
    \includegraphics[width=\linewidth]{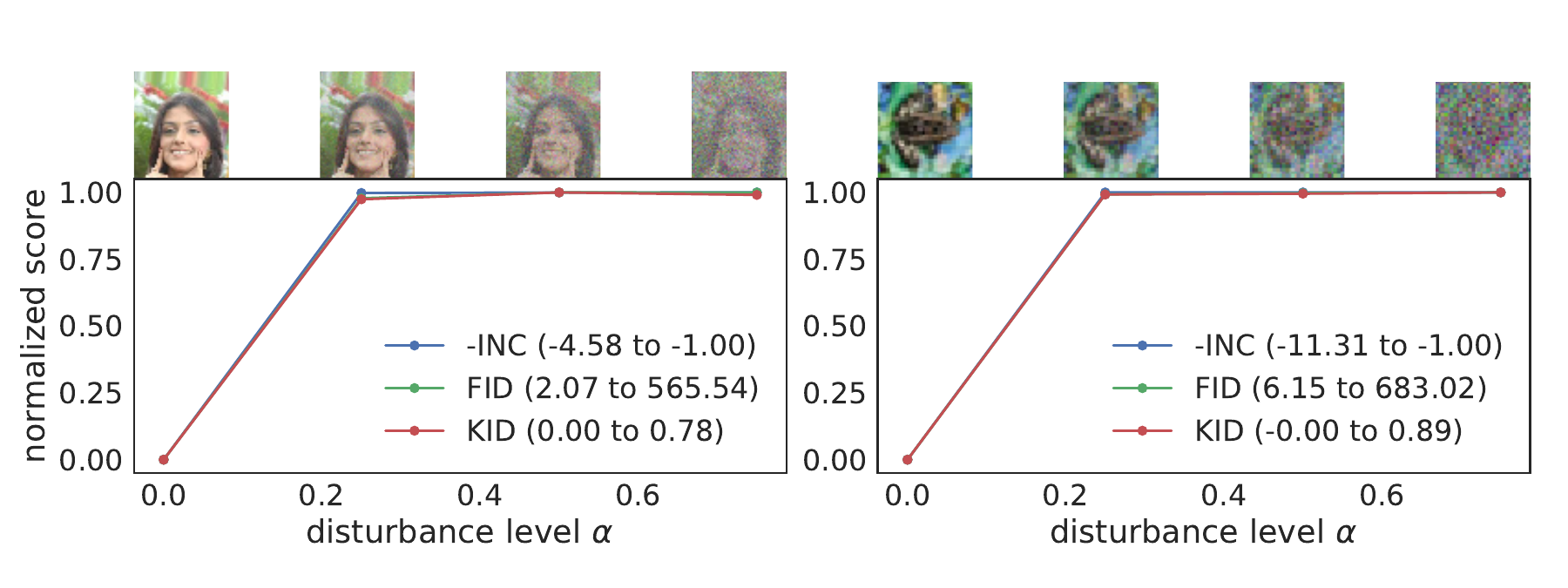}
    \caption{Gaussian noise; $\alpha$ is the mixture weight of the Gaussian noise.}
  \label{fig:disturb:gaussian-noise}
\end{figure}
\begin{figure}%
    \includegraphics[width=\linewidth]{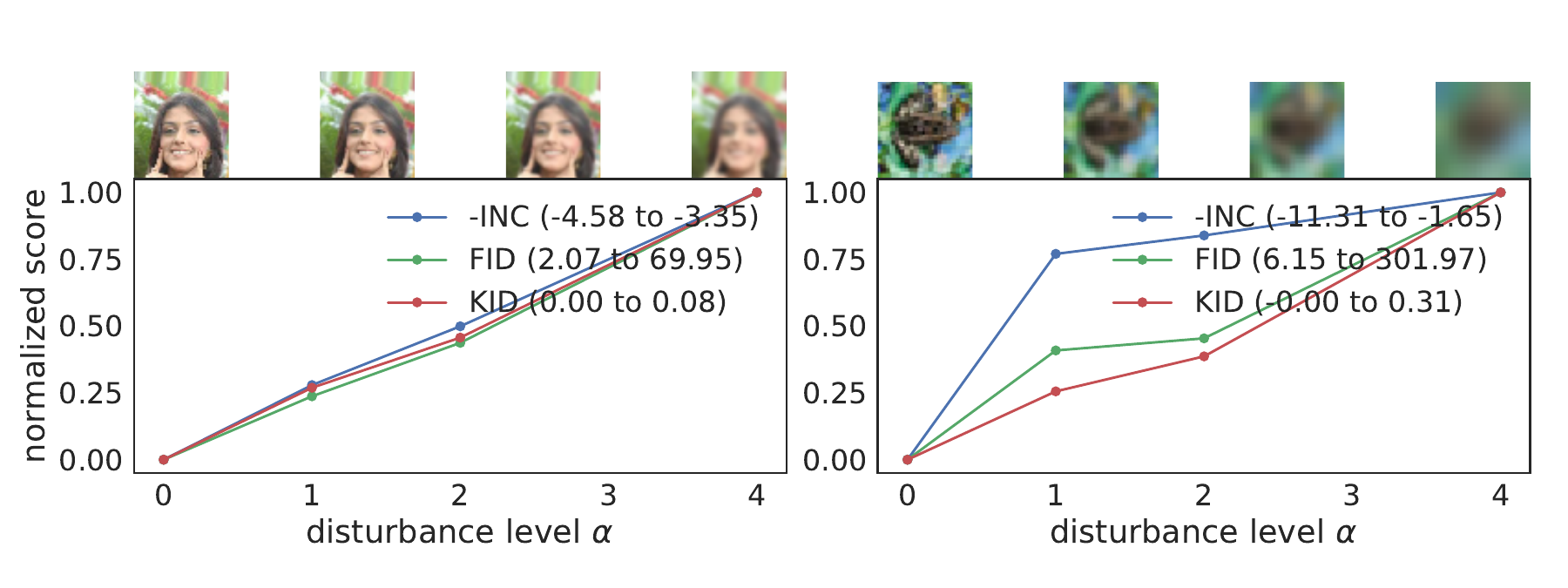}
    \caption{Gaussian blur; $\alpha$ is the standard deviation of the Gaussian filter.}
  \label{fig:disturb:gaussian-blur}
\end{figure}
\begin{figure}%
    \includegraphics[width=\linewidth]{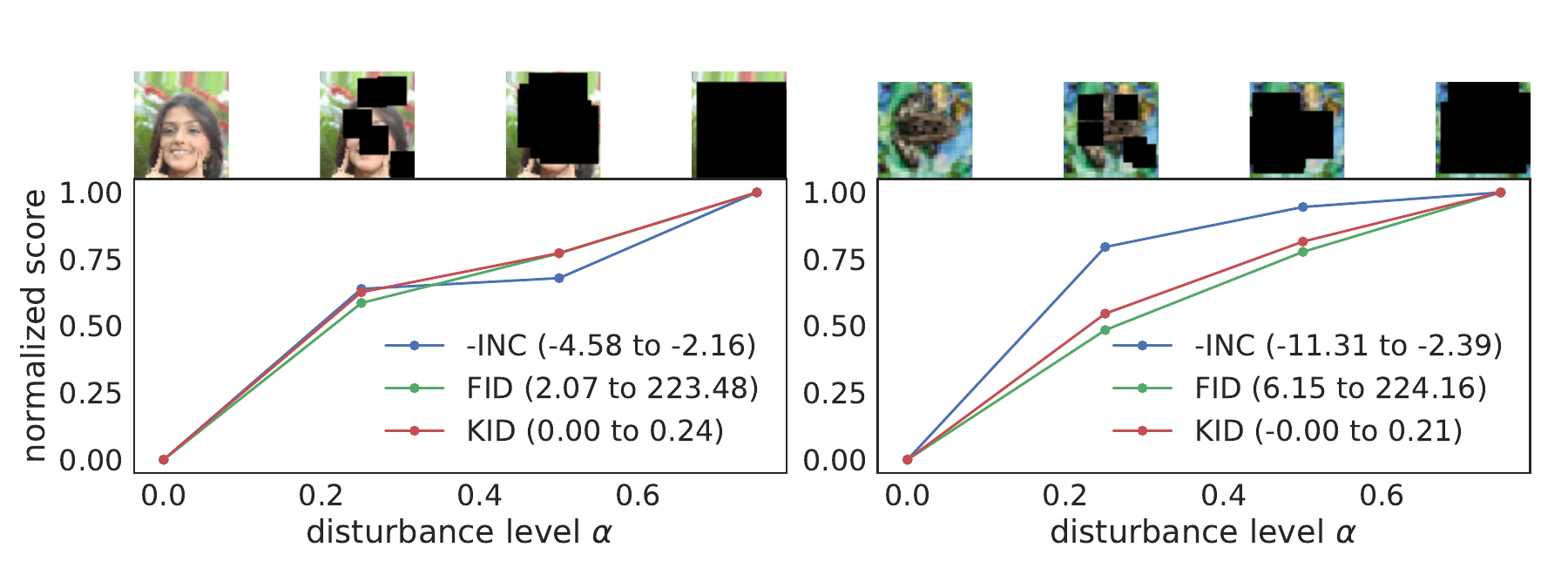}
    \caption{Black rectangles; $\alpha$ is the portion of the image size each rectangle contains.}
  \label{fig:disturb:black-rectangles}
\end{figure}
\begin{figure}%
    \includegraphics[width=\linewidth]{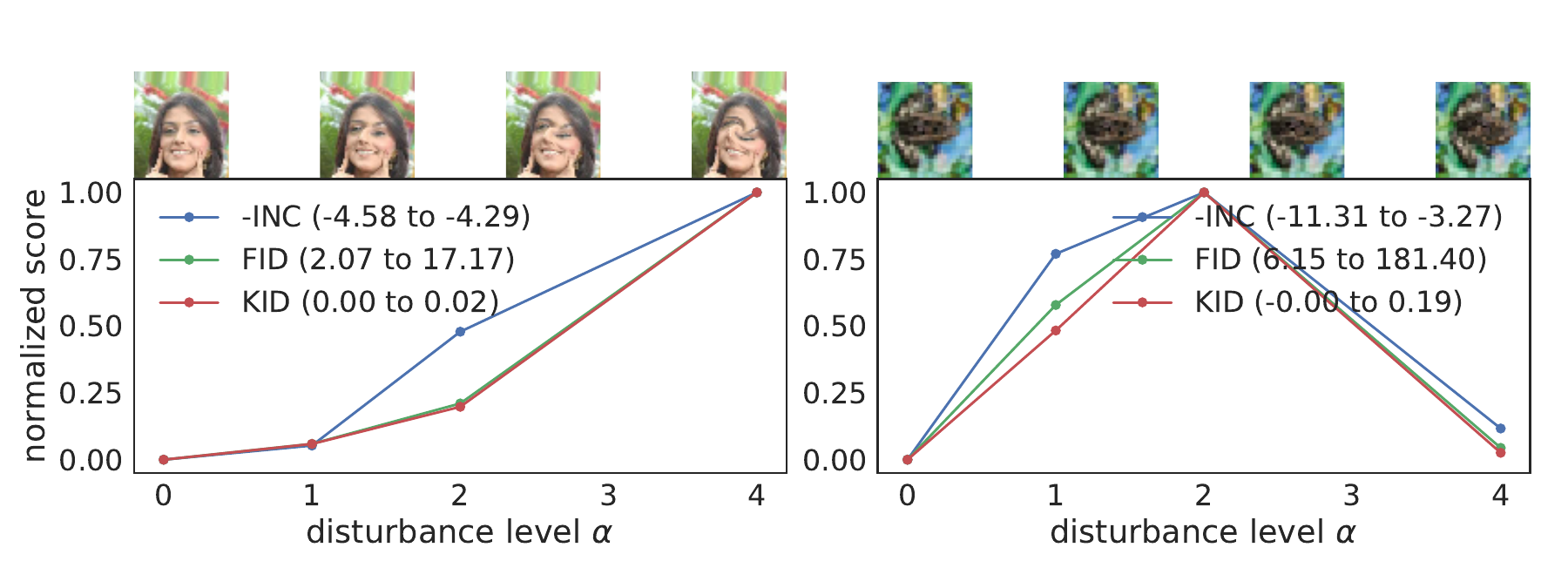}
    \caption{Swirl: $\alpha$ is the strength of the swirl effect.}
  \label{fig:disturb:swirl}
\end{figure}
\begin{figure}%
    \includegraphics[width=\linewidth]{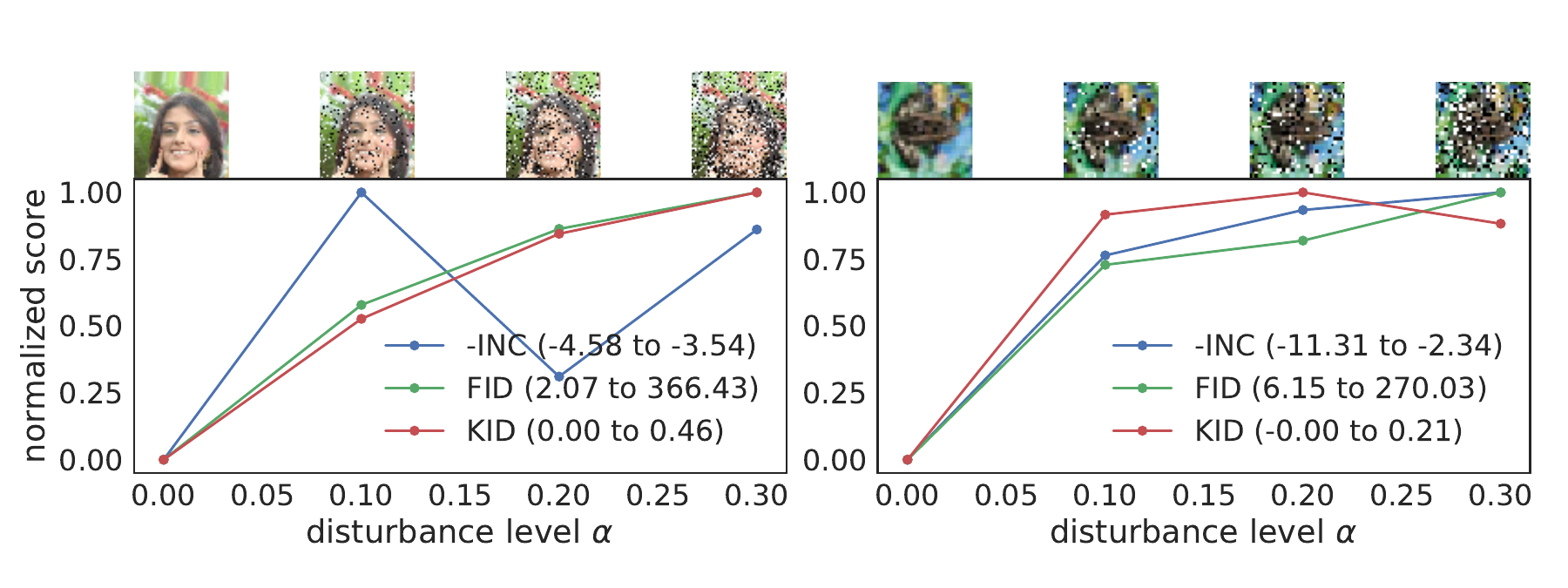}
    \caption{Salt and pepper noise: $\alpha$ is the portion of pixels which are noised.}
  \label{fig:disturb:salt-pepper}
\end{figure}
\begin{figure}%
    \includegraphics[width=\linewidth]{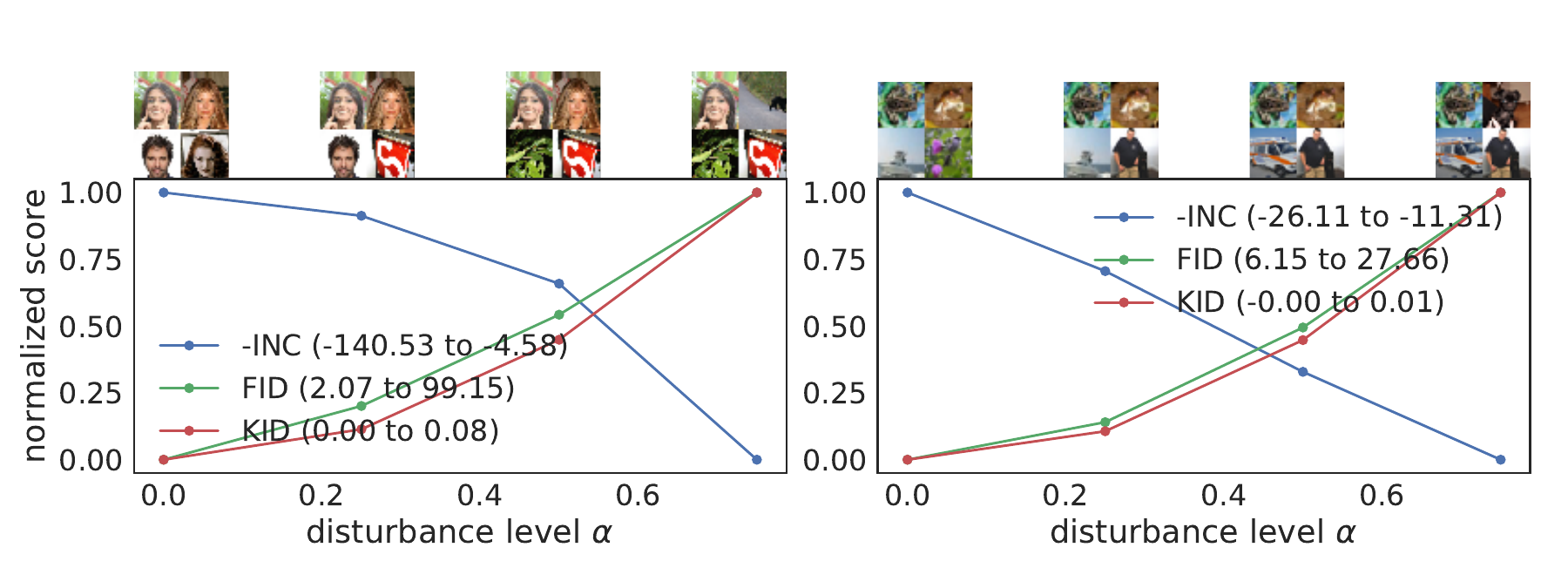}
    \caption{ImageNet contamination: $\alpha$ is the portion of images replaced by ImageNet samples.}
  \label{fig:disturb:imagenet}
\end{figure}

\section{Samples and detailed results for MNIST and CIFAR-10} \label{appendix:samples}
\paragraph{MNIST}
After training for $50\,000$ generator iterations,
all variants achieved reasonable results. Among MMD models, only the distance kernel saw an improvement with more
neurons in the top layer.
\Cref{tab:mnist:scores} shows the quantitative measures,
computed on the basis of a LeNet model.
All have achieved KIDs of essentially zero,
and FIDs around the same as that of the test set,
with Inception scores slightly lower. Model samples are shown in \cref{fig:mnist:samples}.

\begin{table}[ht]
    \centering
    \caption{Mean (standard deviation) of score evaluations for the MNIST models.}
    \label{tab:mnist:scores}
    \begin{tabular}{ccc|rrr}
      & \multicolumn{2}{c|}{critic size} & \multicolumn{3}{c}{} \\
 loss & filters & top layer & \multicolumn{1}{c}{Inception} & \multicolumn{1}{c}{FID} & \multicolumn{1}{c}{KID} \\
\hline
\RQ{}        & 16  &  16 &   9.11  (0.01) &    4.206  (0.05) &   0.005  (0.004)\\
\RBF{}       & 16  &  16 &   8.98  (0.02) &    8.264  (0.02) &   0.011  (0.006)\\
\DOT{}       & 16  &  16 &   8.86  (0.02) &    6.245  (0.06) &   0.006  (0.004)\\
\DIST{}      & 16  & 256 &   9.13  (.004) &    6.179  (0.05) &   0.005  (0.004)\\
Cram\'er GAN & 16  & 256 &   9.25  (0.02) &    3.385  (0.10) &   0.006  (0.005)\\
WGAN-GP      & 16  &   1 &   9.12  (0.02) &    6.915  (0.10) &   0.009  (0.004)\\
test set     & --  &  -- &   9.78  (0.02) &    4.305  (0.16) &   0.003  (0.003)\\
    \end{tabular}
\end{table}

Examining samples during training,
we observed that \RBF{} more frequently produces extremely ``blurry'' outputs,
which can persist for a substantial amount of time before eventually resolving.
This makes sense, given the very fast gradient decay of the \RBF{} kernel:
when generator samples are extremely far away from the reference samples,
slight improvements yield very little reward for the generator,
and so bad samples can stay bad for a long time.

\begin{figure}[ht!]
    \centering
    \begin{subfigure}[t]{0.30\textwidth}
        \centering
        \includegraphics[width=\linewidth]{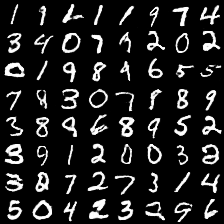}
        \caption{MMD \emph{rq}}
    \end{subfigure}
    \hfill
    \begin{subfigure}[t]{0.30\textwidth}
        \centering
        \includegraphics[width=\linewidth]{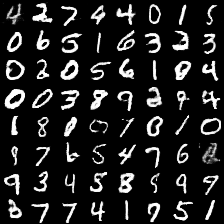}
        \caption{MMD \emph{rbf}}
    \end{subfigure}
    \hfill
    \begin{subfigure}[t]{0.30\textwidth}
        \centering
        \includegraphics[width=\linewidth]{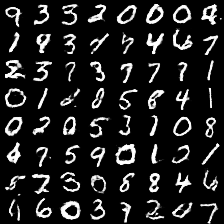}
        \caption{MMD \emph{dot}}
    \end{subfigure}

    \vspace{0cm}

    \begin{subfigure}[t]{0.30\textwidth}
        \centering
        \includegraphics[width=\linewidth]{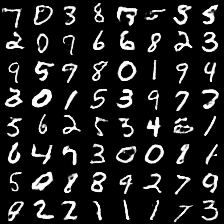}
        \caption{Cramer GAN}
    \end{subfigure}
    \hfill
    \begin{subfigure}[t]{0.30\textwidth}
        \centering
        \includegraphics[width=\linewidth]{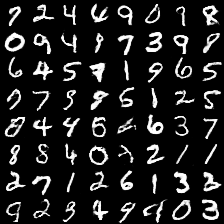}
        \caption{WGAN-GP}
    \end{subfigure}
    \hfill
    \begin{subfigure}[t]{0.30\textwidth}
        \centering
        \includegraphics[width=\linewidth]{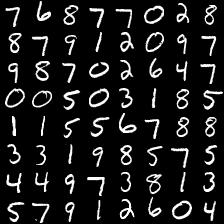}
        \caption{MNIST test set}
    \end{subfigure}
    \caption{Samples from the models listed in \cref{tab:mnist:scores}. Rational-quadratic and Gaussian kernels obtain retain sample quality despite reduced discriminator complexity. Each of these models generates good quality samples with the standard DCGAN discriminator (critic size 64).}
    \label{fig:mnist:samples}
\end{figure}

\paragraph{CIFAR-10}
Scores for various models trained on CIFAR-10 are shown in \cref{tab:cifar-scores}.
The scores for \RQ{} with a small critic network
approximately match those of WGAN-GP with a large critic network,
at substantially reduced computational cost.
With a small critic, WGAN-GP, Cram\'er GAN and the distance kernel all performed very poorly.
Samples from these models are presented in \cref{fig:cifar:samples}.

\begin{table}[ht]
    \centering
    \caption{Mean (standard deviation) of score evaluations for the CIFAR-10 models.}
    \label{tab:cifar-scores}
    \begin{tabular}{ccc|rrr}
      & \multicolumn{2}{c|}{critic size} & \multicolumn{3}{c}{} \\
 loss & filters & top layer & \multicolumn{1}{c}{Inception} & \multicolumn{1}{c}{FID} & \multicolumn{1}{c}{KID} \\
\hline
\RQ{}    & 16   & 16   &    5.86  (0.06) &   48.10  (0.16) &   0.032  (0.001)\\
\RQ{}    & 64   & 16   &    6.51  (0.03) &   39.90  (0.29) &   0.027  (0.001)\\
\DIST{}  & 16   & 256  &    4.53  (0.03) &   80.48  (0.19) &   0.061  (0.001)\\
\DIST{}  & 64   & 256  &    6.39  (0.04) &   40.25  (0.19) &   0.028  (0.001)\\
Cram\'er GAN&16 & 256  &    4.67  (0.02) &   74.93  (0.32) &   0.060  (0.001)\\
Cram\'er GAN&64 & 256  &    6.39  (0.01) &   40.27  (0.15) &   0.028  (0.001)\\
WGAN-GP  & 16   & 1    &    3.15  (0.01) &  147.09  (0.31) &   0.116  (0.002)\\
WGAN-GP  & 64   & 1    &    6.53  (0.02) &   37.52  (0.19) &   0.026  (0.001)\\
test set &  --  & --   &   11.21  (0.13) &    6.11  (0.05) &   0.000  (0.000)\\
    \end{tabular}
\end{table}

\begin{figure}[ht!]
    \centering
    \begin{subfigure}[t]{0.30\textwidth}
        \centering
        \includegraphics[width=\linewidth]{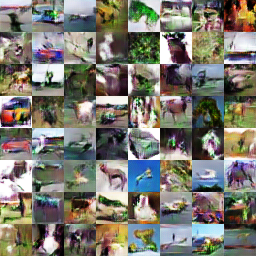}
        \caption{MMD \RQ{}*, critic size 16} \label{fig:cifar:rq-16}
    \end{subfigure}
    \hfill
    \begin{subfigure}[t]{0.30\textwidth}
        \centering
        \includegraphics[width=\linewidth]{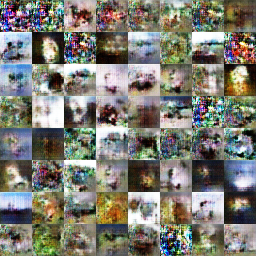}
        \caption{WGAN-GP, critic size 16} \label{fig:cifar:wgan-16}
    \end{subfigure}
    \hfill
    \begin{subfigure}[t]{0.30\textwidth}
        \centering
        \includegraphics[width=\linewidth]{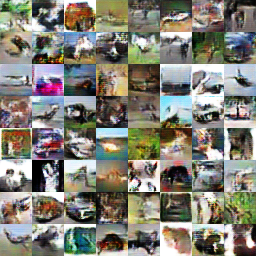}
        \caption{Cram\'er GAN, critic size 16} \label{fig:cifar:cramer-16}
    \end{subfigure}

    \vspace{0cm}

    \begin{subfigure}[t]{0.30\textwidth}
        \centering
        \includegraphics[width=\linewidth]{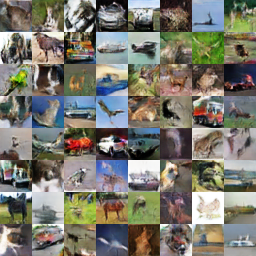}
        \caption{MMD \RQ{}*, critic size 64} \label{fig:cifar:rq-64}
    \end{subfigure}
    \hfill
    \begin{subfigure}[t]{0.30\textwidth}
        \centering
        \includegraphics[width=\linewidth]{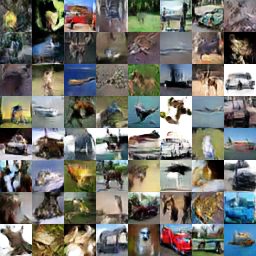}
        \caption{Cram\'er GAN, critic size 64} \label{fig:cifar:cramer-64}
    \end{subfigure}
    \hfill
    \begin{subfigure}[t]{0.30\textwidth}
        \centering
        \includegraphics[width=\linewidth]{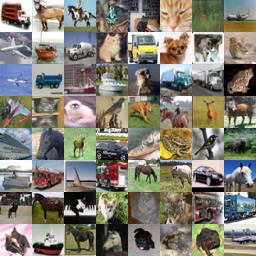}
        \caption{Test set} \label{fig:cifar:test}
    \end{subfigure}
    \caption{Comparison of samples from various models, as well as true samples
from the test set. WGAN-GP samples with critic size 16 are quite bad. Cram\'er GAN
samples with critic size 16 are more appealing to the eye, but seem to have pixel-level
issues. Large-critic Cram\'er and MMD \RQ{}* GAN are of similar quality.}
    \label{fig:cifar:samples}
\end{figure}

\end{document}